\def\year{2022}\relax
\documentclass[letterpaper]{article} 

\usepackage{aaai22}  
\usepackage{caption} 
\DeclareCaptionStyle{ruled}{labelfont=normalfont,labelsep=colon,strut=off} 
\usepackage{courier}  
\usepackage{graphicx} 
\usepackage{helvet}  
\usepackage{natbib}  
\usepackage{times}  
\usepackage[hyphens]{url}  
\usepackage{algorithm}
\usepackage[noend]{algorithmic}
\usepackage{amsmath,amssymb}
\usepackage{amsthm}
\usepackage{bbm}
\usepackage{booktabs}
\usepackage{comment}
\usepackage{enumitem}
\usepackage[utf8]{inputenc}
\usepackage{makeidx}  
\usepackage{multirow}
\usepackage{multicol}
\usepackage{stmaryrd}
\usepackage{subcaption}
\usepackage{svg}
\usepackage{tikz}
\usepackage{xcolor}
\usepackage{xspace}

\usepackage[graphicx]{realboxes}
\usepackage{tikz}
\tikzset{elliptic state/.style={draw,ellipse}}
\usetikzlibrary{shapes,shapes.geometric,arrows,fit,calc,positioning,automata,}
\usetikzlibrary{automata,positioning}
\usepackage{tikz-qtree}
\usepackage{float}
\usepackage{floatflt}

\title{Synthesis from Satisficing and Temporal Goals}

\author {
    Suguman Bansal,\textsuperscript{\rm 1}
    Lydia Kavraki, \textsuperscript{\rm 2}
    Moshe Y. Vardi, \textsuperscript{\rm 2}
    Andrew Wells \textsuperscript{\rm 2,3}
 
}
\affiliations {
    \textsuperscript{\rm 1} University of Pennsylvania\\
    \textsuperscript{\rm 2} Rice University\\
    \textsuperscript{\rm 3} Tesla\\
}

\newcommand{\ltl}{\mathsf{LTL}}
\newcommand{\ltlU}{\mathsf{U}}
\newcommand{\ltlX}{\mathsf{X}}
\newcommand{\ltlNeg}{\neg}
\newcommand{\ltlG}{\mathsf{G}}
\newcommand{\ltlF}{\mathsf{F}}

\newcommand*{\Q}{\mathbb{Q}}
\newcommand*{\Z}{\mathbb{Z}}
\renewcommand*{\L}{\mathcal{L}}
\renewcommand*{\O}{\mathcal{O}}

\renewcommand*{\Re}{\mathbb{R}}

\newcommand{\F}{\mathcal{F}}
\newcommand{\ap}{\mathcal{AP}}
\newcommand{\G}{\mathcal{G}}
\newcommand{\GA}{\mathsf{GA}}

\newcommand*{\A}{\mathcal{A}}
\newcommand*{\Statess}{\mathit{S}}
\newcommand*{\State}{S}
\newcommand*{\Start}{s_I}
\newcommand*{\Final}{\mathcal{F}}

\newcommand*{\init}{v_\mathsf{init}}

\newcommand{\gap}[1]{\mathsf{gap}(#1,d)}
\newcommand{\thresh}{\mathsf{T}}

\newcommand*{\DSum}[2]{\mathit{DS}({#1}, {#2})}

\newcommand*{\R}{\mathsf{R}}

\newcommand{\roundL}[1]{\mathsf{roundLow}(#1,k,p)}

\newcommand{\gapL}[1]{\mathsf{gapLow}(#1,k,p)}

\newcommand{\DSumL}[1]{\mathsf{DSLow}(#1,k,p)}

\newtheorem{theorem}{Theorem}

\newtheorem{lemma}{Lemma}
\newtheorem{definition}{Definition}
\newtheorem*{problem}{Problem}

\begin{document}

\maketitle

\begin{abstract}

Reactive synthesis from  high-level specifications that combine {\em hard} constraints expressed in Linear Temporal Logic ($\ltl$) with {\em soft} constraints expressed by discounted-sum (DS) rewards has applications in planning and reinforcement learning. An existing approach combines techniques from $\ltl$ synthesis with optimization for the DS rewards but has failed to yield a sound algorithm. An alternative approach combining $\ltl$ synthesis with satisficing DS rewards (rewards that achieve a threshold) is sound and complete for integer discount factors, but, in practice, a fractional discount factor is desired. This work extends the existing satisficing approach, presenting the first sound algorithm for synthesis from $\ltl$ and DS rewards with fractional discount factors. The utility of our algorithm is demonstrated on robotic planning domains.

\end{abstract}

\section{Introduction}

Reactive synthesis is the automated construction, from a
high-level description of its desired behavior, of a reactive system that continuously interacts with an uncontrollable external environment~\cite{church1957applications}. 

Recent applications of reactive synthesis have emerged in AI for planning and robotics tasks~\cite{camacho2019towards,he2019efficient,kress2018synthesis}. 
These applications can be formulated as a deterministic turn-based interaction between a controllable system player and an uncontrollable environment player. Given a specification, the synthesis task is to generate a system strategy such that all resulting interactions with the environment satisfy the specification. A large focus in this line of work has been on synthesis from Linear Temporal Logic ($\ltl$) specifications~\cite{pnueli1977temporal,pnueli1989synthesis}.

Yet, several desired specifications either cannot be expressed using $\ltl$ or doing is cumbersome. Examples include specifications about the quantitative properties of systems, such as rewards, costs, degrees of satisfaction, and so on. 
In fact, the combination of $\ltl$ with quantitative properties is used to express more nuanced and complex specifications (see Figure~\ref{fig:fractional_df_req}).
Subsequently, synthesis algorithms from combination specifications have followed~\cite{ding2014optimal,he2017reactive,lahijanian2015time}.


This work investigates the problem of reactive synthesis from specifications that combine {\em hard qualitative constraints} expressed by $\ltl$ with {\em soft quantitative constraints} expressed by {\em discounted-sum rewards}.  Discounted-sum rewards are well-suited for infinite-horizon executions because the discounted-sum is guaranteed to converge on infinite-sequence of costs whereas other aggregation functions such as limit-average may not. Discounted-sum encodes diminishing returns.  As a result, the combination of $\ltl$ with discounted-sum rewards frequently appears in the automated construction of systems using planning and reinforcement learning~\cite{bozkurt2020control,camacho2017non,camacho2019ltl,hasanbeig2019reinforcement,kalagarla2021optimal,kwiatkowska2017prism}. Note, however, these works only deal with a single player case (controllable system agent) while reactive synthesis also assumes the presence of an uncontrollable environment. 

 \begin{figure}[t]
    \centering
    \includegraphics[width = .10\textwidth]{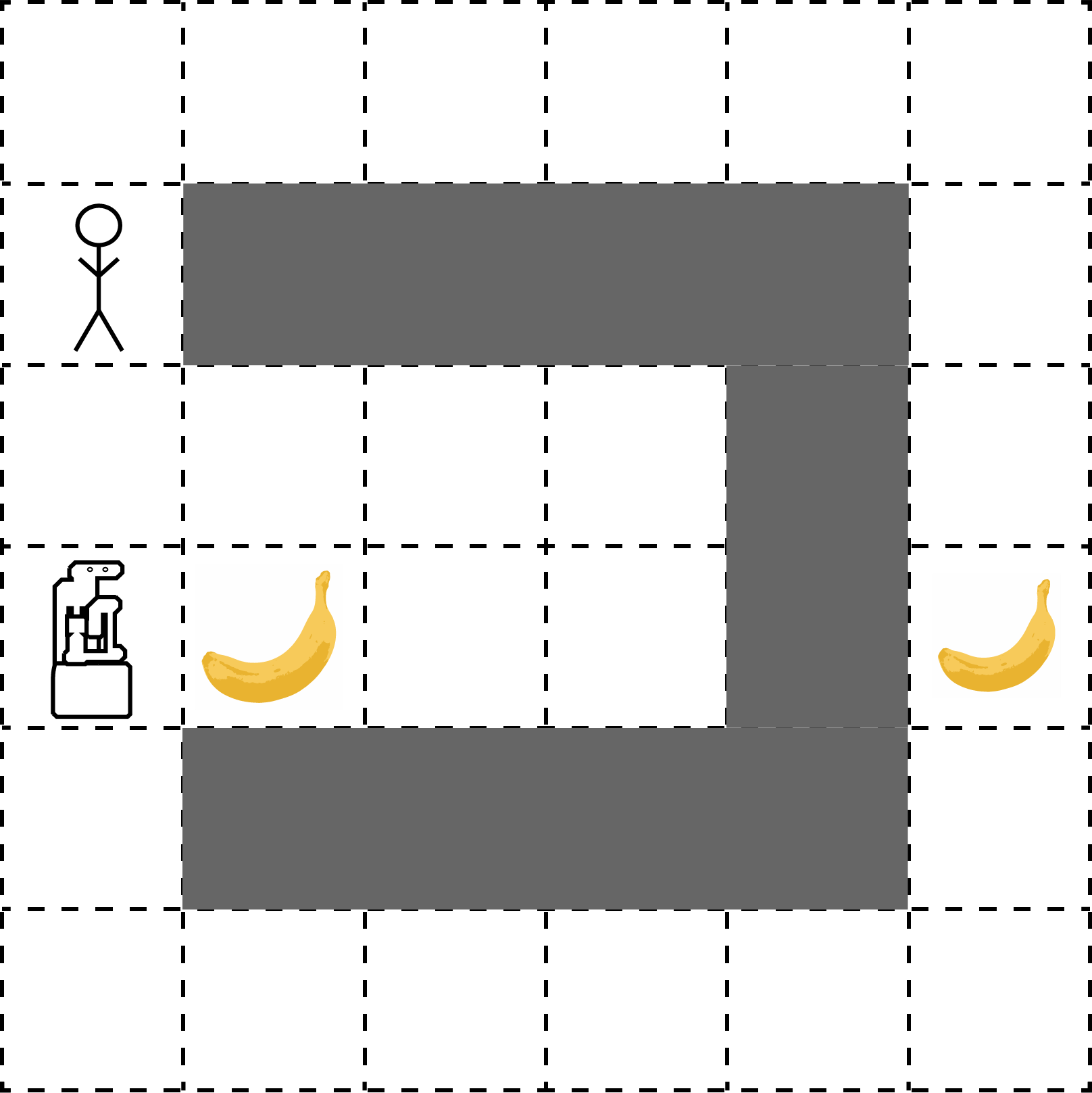}
    \caption{Example scenario: The (controlled) robot must retrieve objects and avoid the (uncontrolled) human in a grocery store. The robot's hard constraint is to retrieve objects from its grocery list without colliding with the walls (in grey) or human. Its soft constraint is to socially (Manhattan) distance itself from the human. A fractional discount factor makes the robot less ``greedy.'' }
    \label{fig:fractional_df_req}
\end{figure}

Broadly speaking, there are two approaches to reactive synthesis from $\ltl$ and discounted-sum rewards. The first approach is based on {\em optimization} of the discounted-sum reward. A strategy that optimizes the discounted-sum reward alone is guaranteed to exist in deterministic, turn-based settings~\cite{shapley1953stochastic}. This existence guarantee, however, is lost upon combination with $\ltl$ constraints. 
For example, consider a two-state game where state $s_0$ gives negative reward and state $s_1$ positive. Each state can transition to all other states. Our $\ltl$ objective is ($Globally\ Eventually\ s_0$). Clearly, there exists no strategy that simultaneously maximizes the discounted-sum reward and satisfies the $\ltl$ objective~\cite{chatterjee2017quantitative}. To this end, an alternate synthesis task is to compute an optimal strategy from those that satisfy the $\ltl$ constraint~\cite{wen2015correct}. Unfortunately, even here existing synthesis algorithms may generate a sub-optimal strategy. 
Overall,  synthesis algorithms from $\ltl$ and discounted-sum rewards that optimize the discounted-sum reward, in one way or another, have hitherto failed to provide guarantees of correctness or completeness.

The second approach to synthesis from $\ltl$ and discounted-sum rewards is based on {\em satisficing} the discounted-sum reward. A strategy is satisficing with respect to a given threshold value $v\in\Q$ if it guarantees the discounted-sum reward of all executions will exceed $v$.  The synthesis task, therefore, is to compute a strategy that satisfies the $\ltl$ specification and is satisficing w.r.t. the threshold value. 
The advantage of this approach is that when the discount factor is an integer, an existing synthesis algorithm is both sound and complete~\cite{BCVTACAS21}. The method builds on novel automata-based technique for quantitative reasoning called {\em comparator automata}~\cite{BCVCAV18,BCVFoSSaCS18}. The central result of comparator automata is that for integer discount factors, examining whether the discounted-sum of an execution exceeds a given threshold reduces to determining the membership of the execution in an (B\"uchi) automaton. Thus, satisficing goals are precisely captured by a comparator. This insight allows for elegant combination of satisficing and temporal goals since both are automata-based.  
The disadvantage of this method is that it cannot be applied with non-integer discount factors since the comparator for non-integer discount factors are not represented by automata. This is a severe limitation because in practice the discount factor is taken to be a fractional value between 1 and 2 in order to reason over a long-horizon~\cite{sutton2018introduction}.  Consider Fig.~\ref{fig:fractional_df_req}. If an integer discount factor greater than 1 is used, the robot will be ``greedy'' and obtain the immediate reward at the cost of becoming ``trapped'' by the human. The fractional discount factor is necessary so the robot recognizes the longer-term benefits to avoid becoming ``trapped.''

The central contribution of this work is a theoretically sound algorithm for reactive synthesis from $\ltl$ and satisficing discounted-sum goals for the case when the discount factor ranges between 1 and 2 (specifically of the form $1+2^{-k}$ for positive integer values of $k$). To the best of our knowledge, this is the first synthesis algorithm from $\ltl$ and discount-sum rewards that offers theoretical guarantees of correctness and is practically applicable.

Our solution is also based on comparator automata. We bypass the issue with fractional discount factors by introducing approximations into the comparator framework. We show that comparators for approximations of discounted-sum with fractional discount factors can be represented by B\"uchi automata.  
In brief, we show that for fractional discount
factors, examining whether the discounted-sum of an execution {\em approximately exceeds} a threshold value can be determined by membership of the execution in a B\"uchi automaton. This combined with synthesis techniques for $\ltl$ gives rise to a purely automata-based algorithm for $\ltl$ and discounted-sum rewards, and thus preserves soundness.

Due to the use of approximation, our algorithm is no longer complete. To this end, we evaluate the practical utility of our algorithm on case studies from robotics planning. Our evaluation demonstrates that our sound but incomplete procedure succeeds in efficiently constructing high-quality strategies in complex domains from nuanced  constraints. 

\section{Preliminaries}
\label{Sec:Prelims}

\subsection{Automata and Formal Specifications}

\subsubsection{B\"uchi Automata and Co-Safety Automata.}
A  {\em B\"uchi automaton} is a tuple
  $\A = (\Statess$, $\Sigma$, $\delta$, $s_\mathcal{I}$, $\Final)$, where
  $ \Statess $ is a finite set of {\em states}, $ \Sigma $ is a finite {\em input alphabet},  $ \delta \subseteq (\Statess \times \Sigma \times \Statess) $ is the   {\em transition relation}, state $ s_\mathcal{I} \in \Statess $ is the {\em initial state}, and $ \Final \subseteq \Statess $ is the set of {\em accepting states}.
A B\"uchi automaton is {\em deterministic} if for all states $ s $ and
inputs $a$, $ |\{s'|(s, a, s') \in \delta \textrm{ for some $s'$} \}|
\leq 1 $.
For a word $ w = w_0w_1\dots \in \Sigma^{\omega} $, a {\em run} $ \rho$ of $ w $ is a sequence of states $s_0s_1\dots$ s.t.
$ s_0 = s_\mathcal{I}$, and $ \tau_i =(s_i, w_i, s_{i+1}) \in \delta $ for all $i$.
Let $ \mathit{inf}(\rho) $ denote the set of states that occur infinitely
often in run ${\rho}$. 
A run $\rho$ is an {\em accepting run} if $ \mathit{inf}(\rho)\cap \Final \neq \emptyset $. A word $w$ is an accepting word if it has an accepting run. 
B\"uchi automata are  closed under set-theoretic union, intersection, and complementation~\cite{thomas2002automata}. 

A {\em co-safety automata} is a deterministic B\"uchi automata with a single accepting state. Additionally, the accepting state is a sink state~\cite{kupferman1999model}.

\subsubsection{Comparator Automata.}


Given an aggregate function $f:\Z^{\omega}\rightarrow \mathbb{R}$, equality or inequality relation $\mathsf{R} \in \{<, >, \leq, \geq, =, \neq\}$, and a threshold value $v \in \Q$, the {\em comparator automaton for $f$ with upper bound $\mu$, relation $\mathsf{R}$, and threshold $v\in\Q$} is an automaton that accepts an infinite word $A$ over the alphabet $\Sigma=\{-\mu,-\mu+1,\cdots \mu\}$ iff $f(A)$ $\mathsf{R}$ $v$ holds~\cite{BCVFoSSaCS18,BCVlmcs2019}. 

The discounted-sum of an infinite-length weight-sequence $W = w_0w_1\dots$ with discount factor $d>1$ is given by $\DSum{W}{d} = \sum_{i=0}^\infty \frac{w_i}{d^i}$.
The comparator automata for the discounted-sum has been shown to be a safety or co-safety automata when the discount factor $d>1$ is an integer, for all values of  $\R$,  $\mu$ and $v$. It is further known to not form a B\"uchi automata for non-integer  discount factors $d>1$, for all values of $\R$, $\mu$ and $v$~\cite{BVCAV19,BCVTACAS21}.

\subsubsection{Linear Temporal Logic.}

{Linear Temporal Logic} ($\ltl$) extends propositional logic with infinite-horizon temporal operators. The syntax of $\ltl$ is defined as $\varphi:= a \in \ap \mid \ltlNeg \varphi \mid \varphi \land \varphi \mid \varphi \lor \varphi \mid \ltlX \varphi \mid \varphi \ltlU \varphi \mid \ltlF \varphi \mid \ltlG \varphi$. 
Here $\ltlX$ (Next), $\ltlU$ (Until), $\ltlF$ (Eventually), $\ltlG$ (Always) are temporal operators.
The semantics of $\ltl$ can be found in~\cite{pnueli1977temporal}.

\subsection{Two-Player Graph Games}

\subsubsection{Reachability Games.}

 A reachability game $G = (V, \init, E, \F)$ consists of a directed graph $(V,E)$, initial state $\init$, and non-empty set of accepting states $\F\subseteq V$.
The set $V$ is partitioned into $V_0$ and $V_1$. 
For convenience, we assume every state has at least one successor. A game is played between two players $P_0$ and $P_1$.

A {\em play} in the game is created by the players moving a token along the edges as follows:
at the beginning, the token is at the initial state. If the token's current position $v$ belongs to $V_i$, then $P_i$ chooses the next position from the successors of $v$. Formally,
a play $\rho = v_0v_1v_2\dots$ is an infinite sequence of states such that $v_0 = v_{\mathsf{init}}$ and  $(v_k, v_{k+1}) \in E$ for all $k\geq 0$. 
A play is  {\em winning for player $P_1$} in the game if it visits an accepting state, and {\em winning for player $P_0$} otherwise.

 A {\em strategy} for a player is a
recipe that guides the player on which state to go next to based on the history of a play.  
A {\em strategy is  winning for a player $P_i$} if for all strategies of the opponent player $P_{1-i}$, all resulting plays are winning for $P_i$. 
 To {solve} a graph game is to determine whether there exists a winning strategy for player $P_1$. 
Reachability games are solved in $\O(|V|+|E|)$~\cite{thomas2002automata}.

\subsubsection{Quantitative Graph Games.}

A { quantitative graph game} (quantitative game, in short) is given by $G = (V  = V_0 \uplus V_1, \init, E, \gamma, \L, d)$ where $V$, $V_0$, $V_1$, $\init$, $E$, plays, and strategies are defined as earlier.  
Each edge is associated with a {\em cost} determined by the cost function $\gamma: E \rightarrow \Z$, and $d>1$ is the {\em discount factor}.
The cost-sequence of a play $\rho$ is the sequence $w_0w_1w_2\dots$ where  $w_k = \gamma((v_k, v_{k+1}))$ for all $i\geq 0$,
The cost of play $\rho$ is the discounted-sum of its cost sequence with discount factor $d>1$.
A  labelling function $\L:V\rightarrow 2^{\ap}$ maps states to propositions from the set $\ap$. The {\em label sequence} of a play $\rho$ is given by $\L(v_0)\L(v_1)\dots$.

\section{Problem Formulation and  Overview}
\label{Sec:Problem}

The two players, the controllable system and uncontrollable environment, interact in a domain described by a quantitative game $G$. The specification of the system player is a combination of hard and soft constraints.

The hard constraint is given as by an $\ltl$ formula $\varphi$. 
A play in $G$ satisfies  formula $\varphi$ if its labelled sequence satisfies the formula. 
We say, a strategy for the system player satisfies a formula $\varphi$ if it guarantees that all resulting plays will satisfy the formula. We call such a strategy {\em $\varphi$-satisfying}.

The soft constraints are given by satisficing goals. W.l.o.g, the system and environment players maximize and minimize the cost of plays, respectively.
Given  a threshold value $v\in\Q$, a play is $v$-satisficing for the system (maximizing) player if its cost is greater than or equal to $v$. Conversely, a play is $v$-satisficing for the environment (minimizing) player if its cost is less than $v$. 
A strategy is $v$-satisficing for a player if it guarantees all resulting plays are $v$-satisficing for the player. 

We are interested in solving the following problem:

\begin{problem} [Reactive Synthesis from Satisficing and Temporal Goals]
Given a quantitative game $\G$, a threshold value $v \in \Q$, and an $\ltl$ formula $\varphi$, the problem of {\rm reactive synthesis from satisficing and temporal goals} is to compute a strategy for the system player that is $\varphi$-satisfying and $v$-satisficing for the player, if such a strategy exists. 
\end{problem}

The problem is solved for integer discount factors~\cite{BCVTACAS21}.

\paragraph{Algorithm Overview.}
In this paper, we extend to fractional discount factors $1<d<2$, yielding practical applications of the synthesis problem. In particular, we solve the problem for $d = 1+2^{-k}$ where $k>0$ is an integer. Since the comparator for discounted-sum with fractional discount factors are not representable by B\"uchi automata, we construct a comparator automata for lower approximations of discounted-sum. This comparator soundly captures the criteria for $v$-satisficing for system player. In particular, if the comparator accepts the weight sequence of a play, then the play must be $v$-satisficing for the player. Therefore, just like $\ltl$ goals, the satisficing goal is also soundly captured by an automaton. Thus, we can reduce the synthesis problem to parity games via appropriate synchronized product constructions of both the automata-based goals. 

The comparator construction for approximation of discounted-sum is presented in Section~\ref{Sec:Comparator} and the reduction to games on graphs is presented in Section~\ref{sec:finalalgorithm}. 

\section{Comparator Construction}
\label{Sec:Comparator}


This section develops the key machinery required to design our theoretically sound algorithm for synthesis from temporal and satisficing goals with fractional discount factors.
We construct comparator automata for a lower approximation of discounted-sum for fractional discount factors of the form $d=1+2^{-k}$ where $k>0$ is an integer. 
We show these comparators are represented by co-safety automata. 

This section is divided in two parts. Section~\ref{sec:definition} defines an aggregate function that approximates the discounted-sum. Section~\ref{Sec:ApproxDSComparator} constructs a comparator for this function. 

Unless stated otherwise, we assume the approximation factor is of the form $\varepsilon=2^{-p}$ where $p>0$ is an integer. Please refer to the Appendix for missing proofs and details.

\subsection{Approximation of Discounted-Sum}
\label{sec:definition}
Given parameters $k,p>0$ of the discount factor and the approximation factor, respectively,  let $\roundL{x}$ be the largest integer multiple of $2^{-(p+k)}$ that is less than or equal to $x$, where $x \in\Re$.  
Let $W[\dots n]$ denote the $n$-length prefix of a weight-sequence $W$

Then, the {\em lower approximation of discounted-sum} of an infinite-length weight-sequence $W$ with discount  factor $d>1$ and approximation factor $\varepsilon>0$ is defined as 
\begin{align*}
    \DSumL{W} &= \lim_{n \rightarrow \infty} \frac{\gapL{W[\dots n]}}{d^{n-1}}
\end{align*}
where the {\em lower gap value} of a finite-length weight-sequence $U$ is defined as 
\begin{align*}
    \gapL{U} = 
    \begin{cases}
    0, \text{ if } |U| = 0 \\
    \mathsf{roundLow}(d\cdot \gapL{V}+v, \\
    ~~~~~~~~~~~k, p), \text{ if } U=V\cdot v
    \end{cases}
\end{align*}

Finally, the definition of $\mathsf{DSLow}$ is completed by  proving $\mathsf{DSLow}$ approximates the discounted-sum of sequences within an additive factor of $d\cdot\varepsilon$:
\begin{theorem}
\label{thrm:ApproxDSLower}
Let $d = 1+2^{-k}$ be the  discount factor and $\varepsilon=2^{-p}$ be the approximation  factor, for rational parameters $p,k>0$. Let $W$ be an infinite-length weight sequence. Then, $ 0\leq \DSum{W}{d} - \DSumL{W} < d\cdot \varepsilon $.
\end{theorem}
\begin{proof}[Proof Sketch]
While the definition of the lower approximation of discounted-sum may look notiationaly dense, it is inspired by an alternate definition of discounted-sum: $$\DSum{W}{d} = \lim_{n\rightarrow \infty} \frac{\gap{W[\dots n]}}{d^{n-1}}$$ where $\gap{U} = 0$ if $|U| = 0$ and $\gap{U} = d\cdot\gap{V} + v$ if $U=V\cdot v$. 

Intuitively, the lower gap value approximates $\mathsf{gap}$. Subsequently, $\mathsf{DSLow}$ approximates the discounted-sum.
\end{proof}

\subsection{Comparator for Approximation of DS}
\label{Sec:ApproxDSComparator}

This section presents the construction of a comparator for lower approximation of discounted-sum defined above. 

\begin{definition}[Comparator automata for lower approximation of DS]
\label{def:comparisonAutLow}
Let $\mu>0$ be an integer bound, and $k,p>0$ be integers.
The {\em comparator automata for lower approximation of discounted sum} with discount factor $d = 1+2^{-k}$, approximation factor $\varepsilon=2^{-p}$, upper bound $\mu$, threshold value $v\in\Q$, and inequality relation $\mathsf{R} \in \{\leq, \geq \}$ is an automaton over infinite weight sequences $W$ over the alphabet $\Sigma = \{-\mu,\dots, \mu\}$ that accepts $W$ iff $\DSumL{W}$ $\mathsf{R}$ $v$. 
\end{definition}

\subsubsection{Construction Sketch.}
We sketch the construction of the comparator for  lower approximation of discounted-sum. 
For sake of exposition, we begin the construction  for threshold value $v = 0$. W.l.o.g., we present  for the relation  $\geq$. Notations $\mu$,  $d=1+2^{-k}$, $\varepsilon=2^{-p}$, and $W$ are from  Definition~\ref{def:comparisonAutLow}.

\newcommand{\upperv}{\mathsf{upperLimit}}
\newcommand{\lowerv}{\mathsf{lowerLimit}}

The lower gap value of prefixes of a weight-sequence can be used as a proxy for acceptance of a weight-sequence in the comparator for the following two observations:
\begin{enumerate}
    \item $\DSumL{W} \geq 0$ for an infinite-length weight sequence $W$  iff there exists a finite prefix $A$ of $W$ such that $\gapL{A} \geq \mu\cdot 2^{k} + 2^{-p}$. Let us denote $\mu\cdot 2^{k} + 2^{-p}$ by $\upperv$.
    \item $\DSumL{W}$ cannot be greater than or equal to 0 iff there exists a finite prefix $A$ of $W$ such that $\gapL{A} \leq -\mu\cdot 2^k$. Let us denote $-\mu\cdot 2^k$ by $\lowerv$.
\end{enumerate}

So, the core idea behind our construction is two fold: (a) use states of the comparator to record the lower gap value of finite-length prefixes, and (b) assign transitions so that the final state of finite-prefix corresponds to its lower gap value.

To this end, we set the initial state to $0$ as the lower gap value of the $0$-length prefix is $0$. The transition relation mimics the inductive definition of lower gap value, i.e. there is a transition from a state $s$ on the alphabet (weight) $a$ to state  $t$ if $t = \roundL{d\cdot s + a}$. 
These ensure that the lower gap value of a finite-state word (finite-length weight-sequence) is detected from the final state in its run. Clearly, the transition relation is deterministic. 

The final piece of the construction is to restrict the automata to finitely many states and to determine its accepting states. 
Note that due to the enumerated observations it is sufficient to track the lower gap value for only as long as it lies between $\lowerv$ and $\upperv$. Observe that there are only finitely many such values of interest since lower gap value is always an integer multiple of $2^{-(p+k)}$. Thus, we have obtained a finite number of states. 
By the first observation, state $\upperv$ is made an accepting sink since every weight-sequence that visits $\upperv$ must be accepted by the comparator. 
Similarly, by the second observation, the state $\lowerv$ is made a non-accepting sink. 
This completes the construction for threshold value $v=0$.

To extend the construction to a non-zero threshold value $v\in\Q$, let $V$ be a {\em lasso} weight-sequence s.t. $\DSum{V}{d}=v$, the comparator incorporates $V$ into its construction. Specifically, we construct a comparator that accepts $W$ iff $\DSumL{W-V}\geq 0$. So, when $W$ is accepted then $\DSum{W}{d}\geq v$, otherwise $\DSum{W}{d}\leq v +d\cdot\varepsilon$.

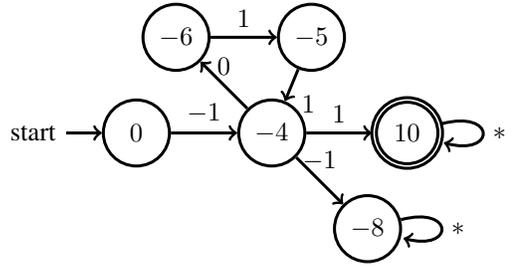
\begin{figure}[t]
\centering
\begin{tikzpicture}[->,auto,node distance=1.8cm,on grid,line width=0.4mm]
          \tikzstyle{round}=[thick,draw=black,circle]
          
  \node[state,initial] (0) {$0$};
  \node[state] (-4) [right of=0] {$-4$};
  \node[state, accepting] (10) [right of=-4] {$10$};
  \node[state] (-6) [above left of=-4] {$-6$};
  \node[state] (-5) [right of=-6] {$-5$};
  \node[state] (-8) [below right of=-4] {$-8$};

  \path (0) edge                node {$-1$} (-4)
        (-4) edge               node {$1$} (10)
        (10) edge [loop right]  node {$*$} (10)
        (-4) edge [above]       node {$0$} (-6)
        (-6) edge               node {$1$} (-5)
        (-5) edge               node {$1$} (-4)
        (-4) edge  [above]      node {$-1$} (-8)
        (-8) edge [loop right]  node {$*$} (-8);

\end{tikzpicture}
\caption{Snippet of comparator for $d=1.5$, $\varepsilon=0.5$, $\mu=1$, $v=0$, and $\geq$. Labels on states have been simplified. A state labelled by $s$ refers to a lower gap value of $s\cdot 2^{-(p+k)}$ }
\label{Fig:comparator}
\end{figure}

\paragraph{}
As an example, Figure~\ref{Fig:comparator} illustrates a snippet of the comparator with discount factor $d=1 + 2^{-1}$, approximation factor $\varepsilon=2^{-1}$, upper bound $\mu = 1$, threshold value $v=0$, and relation $\geq$. As one can see, weight sequence $A = -1, 0, 1^\omega$ with discounted-sum $\frac{1}{3}$ is accepting and weight sequence $B = -1,-1,1^{\omega}$ with discounted-sum $\frac{-1}{3}$ is non-accepting.

\begin{theorem}
\label{thrm:Comparatorlower}
The comparator automata for  lower approximation of discounted sum with discount factor $d = 1+2^{-k}$, approximation factor $\varepsilon=2^{-p}$, upper bound $\mu$, threshold $0$, and inequality relation $\mathsf{R} \in \{\leq, \geq \}$ is a co-safety automata with $\O(\frac{\mu}{(d-1)^2\cdot \varepsilon})$ states, where $k,p>0$ are integers.
\end{theorem}

Therefore, $\DSum{W}{d}\geq 0$ if a weight-sequence $W$ is accepted by the comparator constructed above, and  $\DSum{W}{d} < d\cdot\varepsilon$ otherwise (Theorem~\ref{thrm:ApproxDSLower}-\ref{thrm:Comparatorlower}).

\section{Reactive Synthesis from Satisficing and Temporal Goals}
\label{sec:finalalgorithm}

This section presents the central contribution of this work. We present a theoretically sound algorithm for reactive synthesis from $\ltl$ and satisficing discounted-sum goals for the case when the discount factor ranges between 1 and 2, referred to as fractional discount factors hereon. 

Our algorithm utilizes the comparator automata for the lower approximation of discounted-sum for fractional discount factors constructed in Section~\ref{Sec:Comparator}. For ease of exposition, we present our solution in two parts. First, we present an algorithm for reactive synthesis from satisficing goals only in Section~\ref{Sec:AutomataSatisficing}. Next, we extend this algorithm to solve our original problem in Section~\ref{Sec:finalalgorithm}.

\subsection{Satisficing Goals}
\label{Sec:AutomataSatisficing}

We describe an automata-based solution for reactive synthesis from satisficing goals with fractional discount factor. 
Our solution reduces to reachability games using the comparator for lower approximation of discounted sum. 

The key idea behind our solution is that the said comparator can be treated as a {sufficient} criteria to compute a satisficing strategy for the system player. We explain this further. 
Take a comparator for the lower approximation for discounted-sum with discount factor $d$, approximation factor $\varepsilon$,  threshold $v\in\Q$, and relation $\geq$. Then, a play in the quantitative game is $v$-satisficing for the system player if the comparator accepts the cost sequence of the play. This can be derived directly from  Theorem~\ref{thrm:ApproxDSLower}-\ref{thrm:Comparatorlower}. So, a strategy is $v$-satisficing for the system player if it is winning with respect to the comparator. 
To this end, we construct a {\em  synchronized product} of the quantitative game with the comparator. The resulting product game is a reachability game since the comparator is represented by a co-safety automata. Formally,

\begin{theorem}
\label{thrm:satisficing}
Let $G$ be a quantitative game with discount factor $d = 1+2^{-k}$, for integer $k>0$. Let $v\in Q$ be the threshold value and $\varepsilon=2^{-p}$ be the approximation factor. There exists a reachability game $\GA$ such that
\begin{itemize}
    \item If the system has a winning strategy in $\GA$, then the system has a $v$-satisficing strategy in $G$.
    \item If the environment has a winning strategy in $\GA$, then the environment has a $v+d\cdot\varepsilon$-satisficing strategy in $G$.
\end{itemize}
\end{theorem}
\begin{proof}
The product game synchronizes costs along edges in the quantitative game with the alphabet of the co-safety comparator. 
Let $G = (V  = V_0 \uplus V_1, \init, E, \gamma)$ be a quantitative game.
Let $\mu>0$ be the maximum absolute value of costs along transitions in $G$. 
Then, let $\mathcal{A} = (\State, \Start, \Sigma, \delta, \Final)$ be the co-safety comparator  for lower approximation of discounted-sum with upper bound $\mu$, discount factor $d=1+2^{-k}$, approximation factor $\varepsilon=2^{-p}$, threshold value $v$, and relation $\geq$.  Then, the reachability game is
$\GA = (W = W_0 \uplus W_1, s_0 \times \mathsf{init}, \delta_W, \F_W)$.
Here, $W = V \times S$, $W_0 = V_0 \times S$, and $W_1 = V_1 \times S$. 
Clearly, $W_0$ and $W_1$ partition $W$. 
The edge relation $\delta_W \subseteq W  \times W$ is defined such that edge $((v,s), (v', s')) \in \delta_W$ synchronizes between transitions $(v, v') \in E$ and $(s, a, s') \in \delta$  if $a = \gamma((v, v'))$ is the cost of transition $(v, v')$ in $G$.
State $ s_0 \times \mathsf{init}$ is the initial state and $\F_W = V \times \F$. 

It suffices to prove that a play is winning for the system in $\GA$ iff its cost sequence $A$  in $G$ satisfies $\DSumL{A}\geq 0$. This is ensured by the standard synchronized product construction and Theorem~\ref{thrm:Comparatorlower}. 
The reachability game $\GA$ is linear in size of the  quantitative graph and the comparator. 
\end{proof}

Theorem~\ref{thrm:satisficing} describes a sound algorithm for reactive synthesis from satisficing goals when the discount factor is fractional. The algorithm is not complete since it is possible that there is a $v+d\cdot\varepsilon$-satisficing strategy for the environment even when the system has a $v$-satisficing strategy.

\subsection{Satisficing and Temporal Goals}
\label{Sec:finalalgorithm}

 Finally, we present our theoretically sound algorithm for synthesis from $\ltl$ and discounted-sum satisficing goals for fractional discount factors. 

 The algorithm is essentially a sum of two parts. The algorithm combines the automata-based solution for satisficing goals (presented in Section~\ref{Sec:AutomataSatisficing}) with the classical automata-based solutions for $\ltl$ goals~\cite{pnueli1989synthesis}. Solving satisficing goals forms a reachability game while solving $\ltl$ goals forms a parity game. Thus, the final game which combines both of the goals will be a parity game.  Lastly, the algorithm will inherit the soundness guarantees from both of its parts. 


\begin{theorem}
\label{thrm:finalalgorithm}
Let $G$ be a quantitative game with discount factor $d = 1+2^{-k}$, for integer $k>0$. Let $\varphi$ be an $\ltl$ formula and $v \in \Q$ be  a threshold value. 
Let $\varepsilon=2^{-p}$ be the approximation factor. There exists a parity game $\GA$ such that
\begin{itemize}
    \item If the system has a winning strategy in  $\GA$, then the system has a  $v$-satisficing and $\varphi$-satisfying strategy in $G$.
    \item If the environment has a winning strategy in $\GA$, then then either it has a $v+d\cdot\varepsilon$-satisficing strategy or it has a winning strategy w.r.t. $\ltl$ formula in $G$.  
\end{itemize}
\end{theorem}
\begin{proof}[Proof Sketch]
The reduction consists of two steps of synchronized products: first with the comparator to fulfil the $v$-satisficing goal and then with the automaton corresponding to the $\ltl$ goal. 
The first step conducts the reduction from Theorem~\ref{thrm:satisficing} while lifting the labelling function from the quantitative game to the reachability game: If a state $s$ is labeled by $l$ in the quantitative game, the all states of the form $(s, q)$ will be labelled by $l $ in the reachability game.
The second product synchronizes between the atomic propositions in the reachability game (with a labelling function) and the deterministic parity automaton (DPA)  corresponding to the $\ltl$ specification, thus combining their winning conditions. 

Observe that the product construction is commutative, i.e., one can first construct the product of $G$ with the DPA of the $\ltl$ goal and then with the comparator automata. 

In either case, we generate a parity game of size linear in $|G|$, DPA of the $\ltl$ specification, and the comparator. A winning strategy for the system player in this game is also $v$-satisficing and $\varphi$-satisfying for the same player in $G$.  
\end{proof}

A salient feature of our algorithm is that the complexity to solve the final product game is primarily governed by the temporal goal and not the satisficing goal. 
What we mean is that 
if the temporal goal is given by a fragment of $\ltl$, such as co-safe $\ltl$~\cite{lahijanian2015time}, then the final product game would be reachability game. This is because co-safe $\ltl$ formulas are represented by co-safety automata and thus their combination with comparators would also be a co-safety automata. More generally, if the temporal goal is a conjunction of safety and reachability goals, the resulting game would be a {\em weak-B\"uchi game}, which are also solved in linear time in size of the game~\cite{chatterjee2008linear}. This demonstrates that even though the comparator contributes to growing the state-space of the game linearly, whether the game is solved using efficient linear-time algorithms or higher complexity algorithms for parity games is determined by the temporal goal.  

This feature has implications on the practicality of our algorithm.
In practice, it has been observed that wide-ranging temporal goals in robotics domains can be expressed in simpler fragments and variants of $\ltl$, such as co-safe $\ltl$~\cite{lahijanian2015time} and $\mathsf{LTLf}$~\cite{he2017reactive}. These fragments can be expressed as conjunctions of safety and reachability goals. For this fragment  synthesis from temporal and satisficing goals can be solved in linear-time.

\section{Case Studies}
\label{sec:case_study}

The objective of our case studies is to demonstrate the utility of reactive synthesis from $\ltl$ and satisficing goals in realistic domains inspired from robotics and planning. 
Since ours is the first algorithm to offer theoretical guarantees with fractional discount factors, there are not any baselines to compare to. So, we focus on our scalability trends and identify future scalability opportunities.


\subsection{Design and Set-up}
 We examine our algorithm on two challenging domains inspired from robotic navigation and manipulation problems. Source code and benchmarks are open source\footnote{\url{https://github.com/suguman/NonIntegerGames}}.

\subsubsection{Grid World.}

The robot-human interaction is based on a classic $n\times n$ grid world domain (see Fig~\ref{fig:fractional_df_req}). The grid simulates a grocery store with static obstacles, e.g., placements of aisles. Each agent controls its own location and is only permitted to move in the cardinal directions

\newcommand{\ban}{\mathsf{reach\_banana}}
\newcommand{\wall}{\mathsf{collision\_obstacle}}
\newcommand{\human}{\mathsf{collision\_human}}
\newcommand{\negr}{\mathsf{negative\_reward}}
\newcommand{\posr}{\mathsf{positive\_reward}}

The robot's $\ltl$ constraint is to reach the locations of all items on its grocery list without colliding with the walls (in grey) or the dynamic human, thus combining safety and reachability goals. 
The robot's soft constraints are modelled to achieve two behaviors. The first one is to distance itself from the human. The second is to encode promptness in fulfilling its reachability goal $\ban$.  We model distancing with quantitative rewards using the Manhattan distance between the two agents. Suppose, the locations of the players are $(x_0,y_0)$ and $(x_1,y_1)$, then the reward received by the robot is given by  $\Big\lfloor\frac{\negr}{|x_0-x_1| + |y_0-y_1|}\Big\rfloor$, where $\negr < 0$ is an integer parameter.  We model promptness with an integer $\posr>0$ which the robot receives only when it reaches a location of each item on its grocery list for the first time.  The rewards are additive, i.e., the robot receives the sum of both rewards in every grid configuration. Then, it is reasonable to say that a play accomplishes these two behaviors if the discounted-sum reward of the robot is greater than or equal to 0, i.e., 0-satisficing plays/strategies are good for the robot.

\subsubsection{Conveyor Belt.}
Our second case study is inspired by cutting-edge applications in manipulation tasks~\cite{wells2021finite}. A robot must operate along a $r\times c$ conveyor belt with $r$ rows and $c$ columns across from a human.
When out of reach of the human, the robot can move quickly. Otherwise, it must proceed more slowly. The blocks move down the conveyor belt at a constant speed. Each agent controls the location of its arm. The human also controls the placement of new objects. New blocks are placed whenever a block is removed to maintain a constant number of blocks on the belt at all times. 

The robot's $\ltl$ goal is to avoid interfering with the human. As soft constraints, the robot gains a $\posr$ for every object it grasps and a $\negr$ for every object that falls off the belt. The rewards are additive in every belt configuration. The robot's goal is to ensure its total discounted-sum reward exceeds 0.

\begin{table}[t]
\centering
\begin{tabular}{cc}
\hline
\textbf{Dimensions}                 & \textbf{Number of States} \\ \hline
\multicolumn{2}{c}{\textbf{Grid World}}                       \\ \hline
$n=4$                               & 397                         \\ \hline
$n=6$                               & 2407                         \\ \hline
$n=8$                               & 8093                         \\ \hline
$n=10$                              & 20572                         \\ \hline
\multicolumn{2}{c}{\textbf{Conveyor Belt}}                    \\ \hline
$r\times c = 4\times 3$, ~ 2 blocks   & 9966                         \\ \hline
$r\times c = 5\times 3$, ~ 2 blocks   & 31547                         \\ \hline
$r \times c = 5 \times 3$, ~ 3 blocks & 60540                         \\ \hline
\end{tabular}
\caption{Complexity of Benchmarks: Number of states in product of the labelled quantitative game with the automata of its $\ltl$ specification.}
\label{tab:numofstates}
\end{table}
\paragraph{}

\begin{figure}[t]
    \centering
    \includegraphics[width=0.47\textwidth]{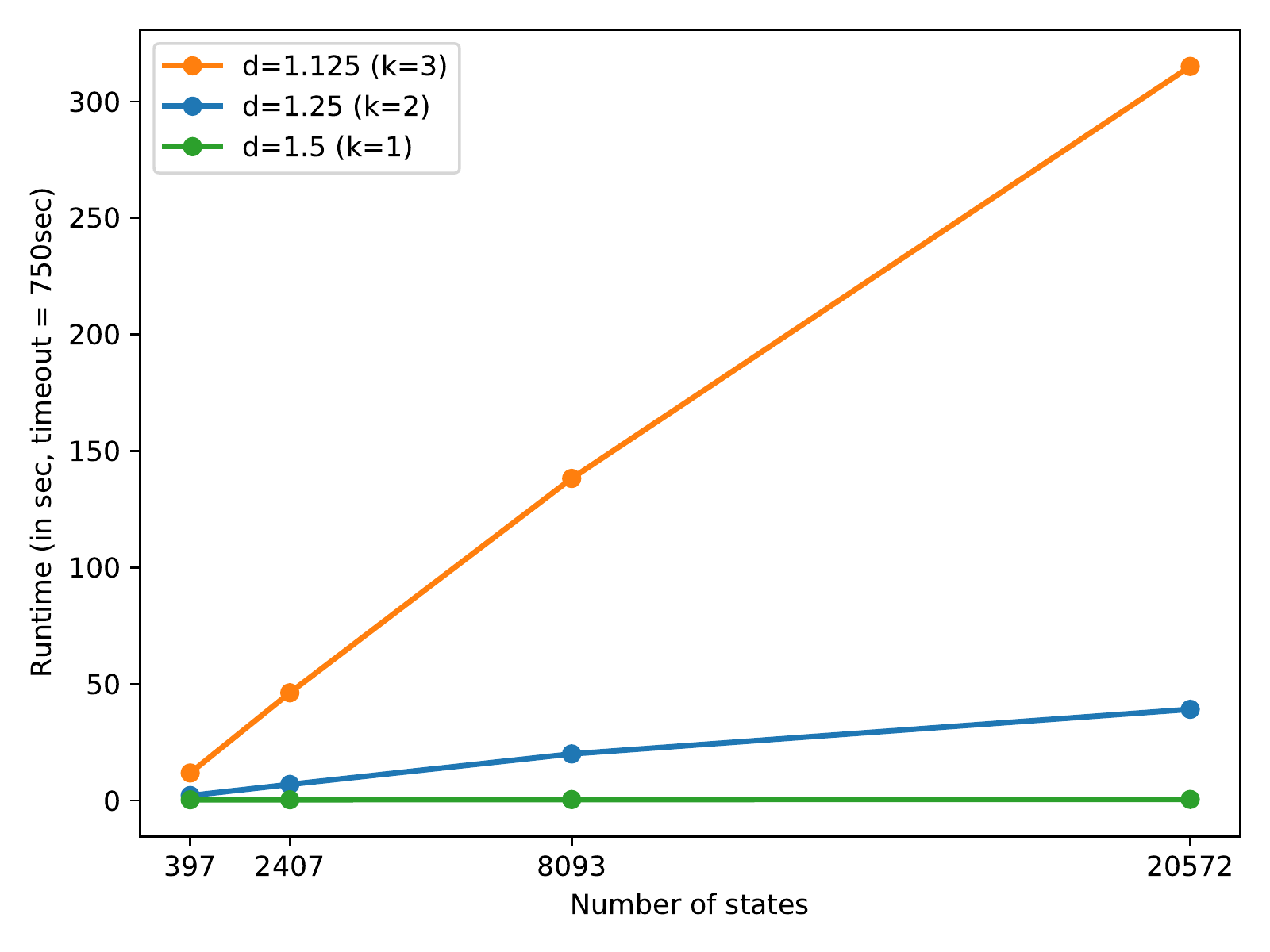}
    \caption{Scalability (Number of states). Plots runtime on Grid World with $\posr=10$ and $\negr = -2$.  $x$-axis are $n = 4,6,8,10$.}
    \label{fig:scaleinstates}
    
\end{figure}

On grid world, we take  $n= 4, 6,8,10$. On conveyor belt, we take $r\times c = 4\times 3, 5\times 3$ with 2 or 3 blocks. 
The hardness of our benchmarks is illustrated  Table~\ref{tab:numofstates}. The benchmarks have so many states since both scenarios have a large number of unique configurations.

Combined with values for $\posr$ and $\negr$, we create 20 grid world and 7 conveyor belt benchmarks. Every benchmark is run with $d = 1.5, 1.25, 1.125$ $(k=1,2,3)$, approx. factor $\varepsilon=0.5$ $(p=1)$ and threshold $v = 0$. 
Our prototype is implemented in $\mathsf{C++}$ on Ubuntu 18.04LTS. Experiments are run on an i7-4770 with 32GBs of RAM with a timeout of 750~$\sec$.

\subsection{Observations}

\begin{figure}[t]
    \centering
    \includegraphics[width=0.47\textwidth]{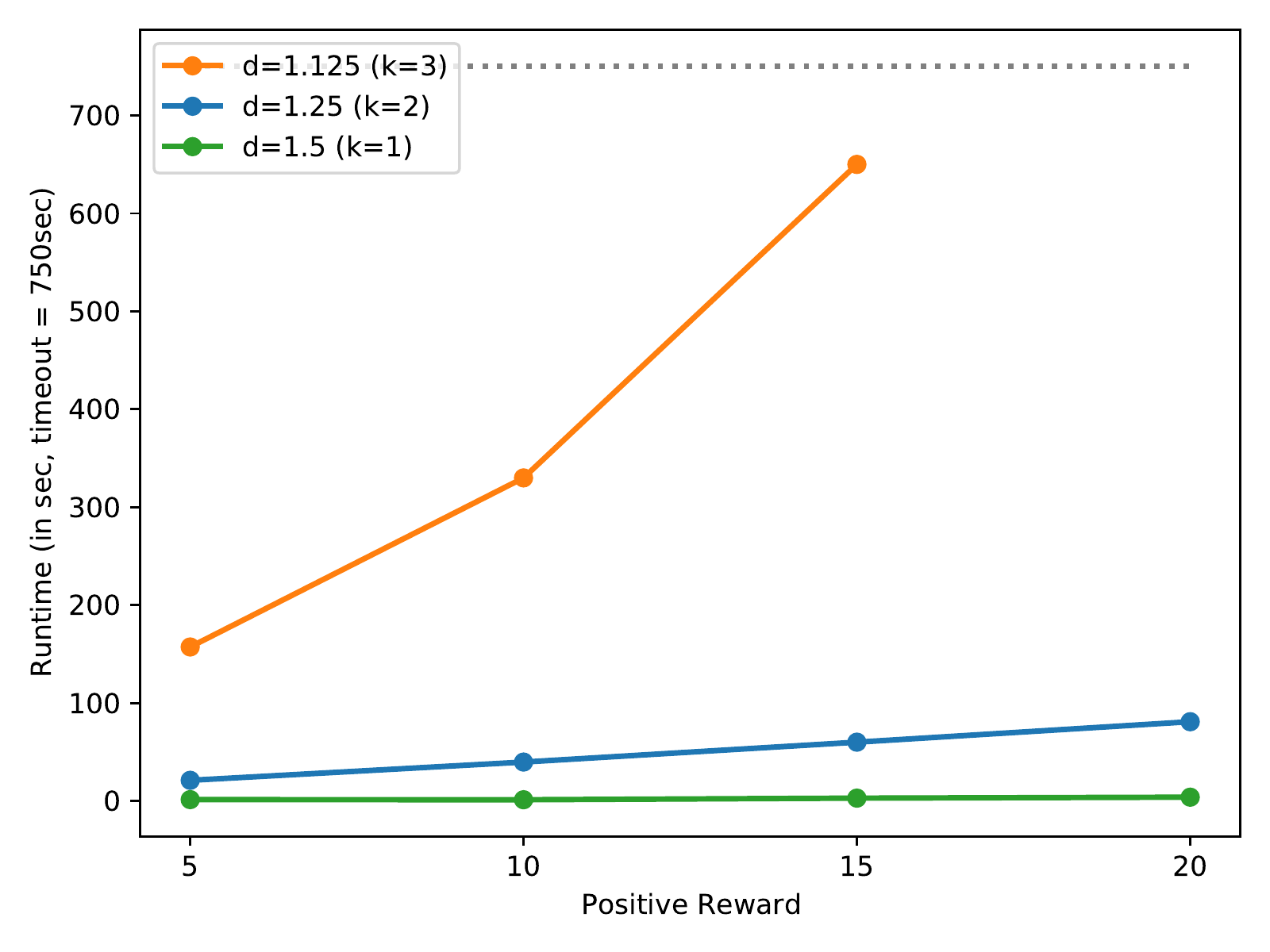}
    \caption{Scalability plot in $\posr$ (affects the size of the comparator). Plotting runtime on Grid world with $n=10$ (20572 states) and $\negr=-2$. }
    \label{fig:scaleinrewards}
\end{figure}

\begin{figure}[t]
    \centering
    \includegraphics[width=0.47\textwidth]{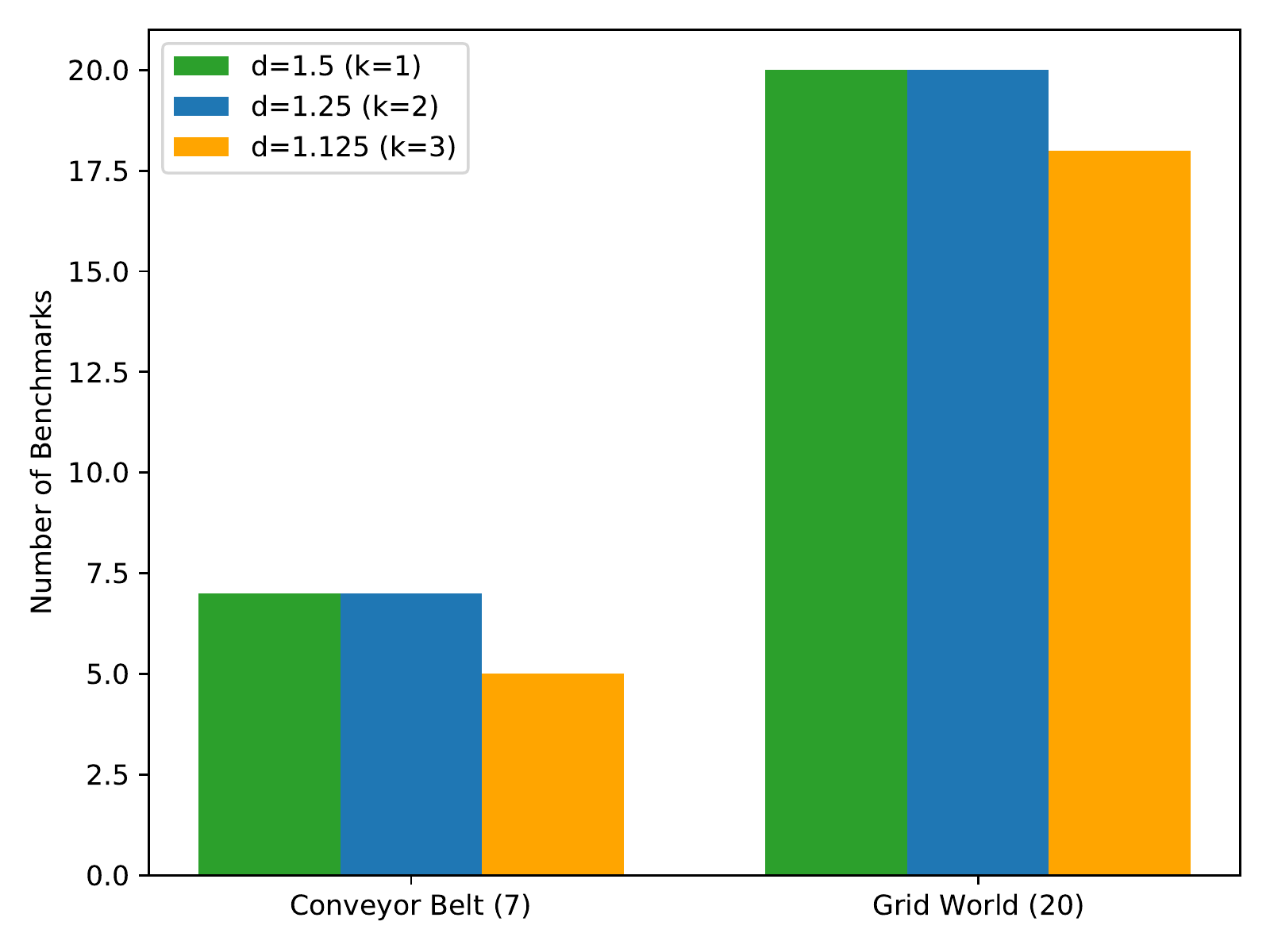}
    \caption{Number of benchmarks solved}
    \label{fig:numsolved}
\end{figure}

Our evaluation demonstrates that our solution successfully scales to very large benchmarks. Despite their difficulty, we solve almost all of our benchmarks (Figure~\ref{fig:numsolved}). Runtime examination indicates that our algorithm is linear in size of the game and the comparator, in practice.   
The scalability trends in size of the game for varying discount factors are shown in Figure~\ref{fig:scaleinstates}. 
Determining the dependence on the comparator automata is more involved since its size depends on several parameters, namely $\posr$, the discount factor, and the approximation factor. Figure~\ref{fig:scaleinrewards} suggests the algorithm is linear in $\posr$. The margin between the three discount factor curves on Fig~\ref{fig:scaleinstates}-\ref{fig:scaleinrewards} suggests a significant blow-up as the discount factor nears 1. Additional experiments (see Appendix) that vary the approximation factor also display a significant blow-up as the approximation factor decreases. These are not alarming since the size of the comparator is in the order of $\O(\posr)$, $\O((d-1)^{-2})$ and $\O(\varepsilon^{-1})$. These reflect that our current implementation is faithful to the theoretical analysis of the algorithm. 

These are encouraging results as our implementation uses explicit state representation. The overhead of this state representations can be very high. In some cases, we observed that for the large benchmarks about 70\% of the total compute time may be spent in constructing the product explicitly. Despite these issues with explicit-state representation, our  algorithm efficiently scales to large and challenging benchmarks. 
This indicates  potential for further improvements. 

In terms of quality of solutions, the resulting strategies are of better quality. For example, in Figure~\ref{fig:fractional_df_req} we observed that as the discount factor becomes smaller the robot is able to reason for a longer horizon and not get "trapped". Another benefit are the soundness guarantees. They are especially valuable in environments such as the Conveyor belt which are so complex that they preclude a manual analysis.

To conclude, our case studies demonstrates the promise of our approach in terms of its ability to scale and utility in practical applications, and encourage future investigations.

\section{Conclusion}

Combining hard constraints (qualitative) with soft constraints (quantitative) is a challenging problem with many applications to automated planning. 
This paper presents the first sound algorithm for reactive synthesis from $\ltl$ constraints with soft discounted-sum rewards when the discount factor is fractional. Our approach uses an automata-based method to solve the soft constraints, which is then elegantly combined with existing automata-based methods for $\ltl$ constraints to obtain the sound solution. Case studies on classical and modern domains of robotics planning demonstrate use cases, and also, shed light on recommendations for future work to improve scalability to open up exciting applications in robotics e.g. warehouse robotics, autonomous driving, logistics (supply-chain automation).

\section*{Acknowledgements}
We thank anonymous reviewers. This work is supported in part by NSF grant 2030859 to the CRA for the CIFellows Project, NSF grants IIS-1527668, CCF-1704883, IIS-1830549, CNS-2016656, DoD MURI grant N00014-20-1-2787, and an award from the Maryland Procurement Office.

\bibliography{main}

\appendix
\section{Appendix: Complete Proofs}

\subsection{Definition of lower approximation of DS is well-defined}

For an infinite-length weight sequence $W$, let $W[\dots n]$ denote its $n$-length prefix for $n\geq 0$. Given parameters $k$ and $p$ of the discount factor and the approximation factor, respectively, let the {\em resolution} be given by $r = 2^{-(p+k)}$.
 For real number $x \in\Re$, let $\roundL{x}$ denote the largest integer multiple of the resolution that is less than or equal to $x$. Formally, $\roundL{x} = i\cdot 2^{-(p+k)}$ for an integer $i \in \Z$  such that for all integers $j\in \Z$ for which $j\cdot 2^{-(p+k)} \leq x$, we get that $j\leq i$. Then it is clear that for all real values $x \in \Re$, $0\leq x - \roundL{x} < 2^{-(p+k)}$.
Then, the lower approximation of discounted-sum is defined as follows:

\begin{definition}[Lower Approximation of Discounted-Sum]
\label{def:approxDSLow}
{\em Given discount factor  $d = 1+2^{-k}$ and approximation factor $\varepsilon=2^{-p}$ with rational-valued parameters $k,p \in \Q$. The {\em lower gap} of a finite-length weight sequence $U$, denoted $\gapL{U}$ is 0 if $|U| = 0$ and $\roundL{\gapL{V}+v}$ if $U = V\cdot v$.
Then, the {\em lower approximation of discounted sum} of an infinite-length weight sequence $W$ with discount  factor $d$ and approximation factor $\varepsilon$ is denoted by and defined as follows:
\begin{align*}
    \DSumL{W} &= \lim_{n \rightarrow \infty} \frac{\gapL{W[\dots n]}}{d^{n-1}}~\label{Eq:LowerApprox}
\end{align*}
}
\end{definition}

Our \textbf{goal} is to show that $\DSumL{W} = \lim_{n\rightarrow\infty}\frac{\gapL{W[n]}}{d^{n-1}}$ is well-defined, i.e., the limit of $\frac{\gapL{W[\dots n]}}{d^{n-1}}$ exists as $n\rightarrow \infty$ (Theorem~\ref{thrm:LimitExistsLower}). Next, we need to show that Definition~\ref{def:approxDSLow} indeed computes a value that approximates the discounted-sum of a sequence (Theorem~\ref{thrm:ApproxDSLower}).

We begin with some additional notation.
Let $d >1$ be a rational valued discount factor. 
The {\em recoverable gap} of a finite, bounded, weight-sequence $U$ and discount factor $d$, denoted by $\gap{U}$, is 0 if $|U|=0$ and $\gap{V\cdot v} = d\cdot \gap{V} + v$ if $U = V\cdot v$.
Intuitively, the recoverable gap of a finite weight-sequence is the normalized discounted-sum of the finite weight-sequence. Then, it is known that for an infinite-length weight sequence  $W$ $\lim_{n\rightarrow\infty}\frac{\gap{W[\dots n]}}{d^{n-1}} \rightarrow \DSum{W}{d}$.
Then the following holds:

\begin{lemma}
\label{lem:EquivGapAndApproxGapLow}
Let $d = 1+2^{-k}$ and $2^{-p}$ be the discount-factor and precision, for  rational numbers $k,p>0$. Let $\mu>0$ be the upper-bound.
Let $W$ be an infinite and bounded  weight-sequence. Then, there exists an infinite and bounded {\em rational number} weight-sequence $U$ such that for all $n>0$, $\gapL{W[\dots n]} = \gap{U[\dots n]}$.
\end{lemma}

\begin{proof}
For sake of simplicity, we assume $W$ is an integer weight sequence. The proof extends to rational weight sequences as well. 
Let $W = w_0w_1w_2\dots$ such that for all $i\geq 0$, $w_i \in \Z$ and $|w_i| < \mu$. We will construct the desired infinite-length weight sequence $U$ inductively. 

\paragraph{Base Case.} Consider the 1-length prefix of $W$, $W[\dots1] = (w_0)$. By definition,  $\gapL{W[\dots 1]} = \roundL{(w_0)} = w_0$. So, we set $u_0$, the 0-th element of $U$, to be $w_0$. Clearly, $ \gap{U[\dots 1]} = w_0 = \gapL{W[\dots 1]}$.

\paragraph{Inductive Hypothesis.} For an $n>0$, let there exist an $n$-length rational-number sequence $ (u_0u_1\dots u_{n-1})$ bounded by $\mu$ such that for all $m\leq n$ it holds that $\gapL{W[\dots m]} = \gap{(u_0u_1\dots u_{m-1})}$.

\paragraph{Induction Step.} 
It suffices to prove that the $n$-length weight-sequence $ (u_0u_1\dots u_{n-1})$ 
can be extended by appending a rational-number $u_n$ bounded by $\mu$ such that 
$\gapL{W[\dots (n+1)]} = \gap{(u_0u_1\dots u_{n})}$ holds. 

By definition, $\gapL{W[\dots (n+1)]} =\roundL{d\cdot \gapL{W[\dots n]} + w_n}$.
By definition of $\mathsf{roundLow}$, there exists a $0\leq \varepsilon_n < 2^{-(p+k)}$ such that $\roundL{d\cdot \gapL{W[\dots n]} + w_n} = d\cdot \gapL{W[\dots n]} + w_n -\varepsilon_n$. 
Therefore, we obtain  $\gapL{W[\dots (n+1)]} = d\cdot \gapL{W[\dots n]} + w_n -\varepsilon_n$. By I.H., we see $\gapL{W[\dots (n+1)]} = d\cdot \gap{u_0u_1\dots u_{n-1}} + w_n -\varepsilon_n$. Set $u_n = w_n -\varepsilon_n$. Then,  we obtain that $|u_n|\leq \mu$. Therefore, $\gapL{W[\dots (n+1)]} = \gap{u_0u_1\dots u_n}$.

Therefore, let $U$ be the infinite and bounded rational-number weight-sequence generated as defined above. Then  for all $n>0$, $\gapL{W[\dots n]} = \gap{U[\dots n]}$.

Note that such a $U$ exists for all infinite and bounded-weight sequences $W$, even if $W$ is not an integer weight-sequence. The same proof can be replicated for that case as well. The difference is that for a general rational number weight sequence if $W$ is bounded by $\mu$, then the $U$ will be bounded by $\mu+1$.
\end{proof}

\begin{theorem}
\label{thrm:LimitExistsLower}
Let $d = 1+2^{-k}$ and $2^{-p}$ be the discount-factor and precision, for  rational numbers $k,p>0$. Let $\mu>0$ be the upper-bound.
Let $W$ be an infinite and bounded weight-sequence. 
Then  $\lim_{n\rightarrow\infty}\frac{\gapL{W[\dots n]}}{d^{n-1}}$ exists, where $W[\dots n]$ is the $n$-length prefix of $ W$. 
\end{theorem}

\begin{proof}
We know from Lemma~\ref{lem:EquivGapAndApproxGapLow}, that there exists an infinite and bounded {\em rational number} weight-sequence $U$ such that for all $n>0$, $\gapL{W[\dots n]} = \gap{U[\dots n]}$. Therefore, $\frac{\gapL{W[\dots n]}}{d^{n-1}} = \frac{\gap{U[\dots n]}}{d^{n-1}}$. Since $\lim_{n\rightarrow\infty}\frac{\gap{U[\dots n]}}{d^{n-1}}$ exists, we also get that  $\lim_{n\rightarrow\infty}\frac{\gapL{W[\dots n]}}{d^{n-1}}$ exists and it is equal to $\DSum{U}{d}$.
\end{proof}

We have proven that the desired limit exists. Therefore, Definition~\ref{def:approxDSLow} is well-defined.

Next, we prove that Definition~\ref{def:approxDSLow} computes a value that approximates the discounted-sum of a weight sequence. 
In the following, we will define the {\em  resolution sequences} as follows:
An $n$-length resolution sequence is the $n$-length sequence in which all elements are the resolution $r=2^{-(p+k)}$.

\begin{lemma}
\label{lem:LowerGapDifBounded}
Let $d = 1+2^{-k}$ and $2^{-p}$ be the discount factor and approximation factor, for  rational numbers $k,p>0$. Let $\mu>0$ be the upper-bound.
Let $W$ be a non-empty, finite-length, and bounded weight sequence. Then,
\[0\leq \gap{W} - \gapL{W} < \gap{R}\]
where $R$ is $|W|$-length resolution sequence. 
\end{lemma}

\begin{proof}
The proof proceeds by induction on the length of the weight sequence.

\paragraph{Base Case.}
When $|W|=1$. Let $W = w_0$ where $w_0\in\mathbb{Z}$ and $|w_0|\leq \mu$. Then $\gap{W} = \gapL{W} = w_0$. Then $\gap{W} = W_0$ and $\gapL{W} = \roundL{W_0}$. 
Thus, trivially, 
$0\leq \gap{W} - \gapL{W} < 2^{-(p+k)} = \gap{R}$, where $R$ is the resolution sequence of length $1$.

\paragraph{Inductive Hypothesis.}
For all weight-sequences $W$ of length $n\geq1$,  it is true that $0\leq \gap{W} - \gapL{W} < \gap{R}$, where $R$ is $|W|$-length resolution sequence.

\paragraph{Induction Step.} We extend this result to weight-sequences of length $n+1$. Let $W$ be an $n+1$-length weight-sequence. Let $W = W[\dots n] \cdot w_n$ $w_n\in\Z$ such that $|w_n|< \mu$. 

First, we show that $\gap{W} - \gapL{W}\geq 0$:
\begin{align*} 
& \gap{W} - \gapL{W} \\
= &  d\cdot \gap{W[\dots n]} + w_n \\ & - \roundL{d\cdot\gapL{W[\dots n]} + w_n }\\
& \text{ From the I.H. we get} \\
\geq &  d\cdot \gap{W[\dots n]} + w_n \\
& - \roundL{d\cdot\gap{W[\dots n]} + w_n }
\end{align*} 
Since $\gap{a}-\gapL{a}\geq 0$, we obtain the desired result that $\gap{W} - \gapL{W}\geq 0$.

Next, we show that $\gap{W} - \gapL{W} < \gap{R}$, where $R$ is the $|W|$-length resolution sequence. 
\begin{align*} 
& \gap{W} - \gapL{W} \\
= &  d\cdot \gap{W[\dots n]} + w_n \\
& - \roundL{d\cdot\gapL{W[\dots n]} + w_n }\\
& \text{Since $\gap{a} - \gapL{a}< 2^{-(p+k)}$, we get} \\
< &     d\cdot \gap{W[\dots n]} + w_n \\
&- {(d\cdot\gapL{W[\dots n]} + w_n)} + 2^{-(p+k)} \\
= &   d\cdot \gap{W[\dots n]}  - {d\cdot\gapL{W[\dots n]}} \\
& + 2^{-(p+k)} \\
& \text{ From the I.H. we obtain} \\
< & d\cdot\gap{R'}  + 2^{-(p+k)} \\
& \text{ where } R' \text{ is the } n\text{-length resolution sequence} \\
= & \gap{R} \text{ where } R \text{ is the } (n+1)\text{-length resolution sequence}
\end{align*} 
This concludes our proof. 
\end{proof}

\subsection{Proof of Theorem~\ref{thrm:ApproxDSLower}}

\paragraph{Theorem~\ref{thrm:ApproxDSLower}.}
{\em Let $d = 1+2^{-k}$ be the  discount factor and $\varepsilon=2^{-p}$ be the approximation  factor, for positive rational parameters $p,k>0$. Let $W$ be an infinite-length weight sequence. Then,
 $$ 0\leq \DSum{W}{d} - \DSumL{W} < d\cdot \varepsilon $$
}

\begin{proof}
\sloppy
Let $R_n$ denote the $n$-length resolution sequence, and $R$ be infinite-length resolution sequence. From Lemma~\ref{lem:LowerGapDifBounded}, we know that for all $n>0$, 
\begin{align*}
&0\leq \gap{W[n]} - \gapL{W[n]} \\
& < \gap{R_n}\\
\iff & 0\leq \frac{(\gap{W[n]} - \gapL{W[n]})}{d^{n-1}} \\
& < \frac{\gap{R_n}}{d^{n-1}}\\
\iff &0\leq \frac{(\gap{W[n]}}{d^{n-1}} - \frac{\gapL{W[n]})}{d^{n-1}} \\
& < \DSum{R}{d} \\
&\text{By taking the limit and by further simplification, we get} \\
\iff &0\leq \DSum{W}{d} - \DSumL{W} < d \cdot \varepsilon
\end{align*}
\end{proof}

\subsection{Comparator Automata Construction}
\paragraph{Theorem~\ref{thrm:Comparatorlower}. }
{\em Let $\mu>0$ be and integer upper bound. Let $k,p>0$ be {\em integer} parameters s.t. $d = 1+2^{-k}$ is the  discount factor and $\varepsilon=2^{-p}$ is the approximation parameter. Then, the comparator automata for  lower approximation of discounted sum with discount factor $d = 1+2^{-k}$, approximation factor $\varepsilon=2^{-p}$, upper bound $\mu$, threshold 0 and inequality relation $\mathsf{R} \in \{\leq, \geq \}$ is $\omega$-regular.
}

\begin{proof}
The proof presents the construction of a co-safety automaton for the said comparator, thus proving the comparator is $\omega$-regular.
Recall, the parameters are  integer upper bound $\mu>0$,  discount factor $d = 1+2^{-k}$, and approximation factor $\varepsilon=2^{-p}$ where $k,p>0$ are integer discount factors, and threshold value is 0. We present the construction  for relation $\geq$. The relation $\leq$ follows a similar construction. 

Let $\thresh_l$ be the largest integer such that $\thresh_l\cdot 2^{-(p+k)} \leq -\mu \cdot 2^{k}$. Let $\thresh_u$ be the smallest integer such that
$\thresh_u \cdot 2^{-(p+k)} \geq \mu\cdot 2^{k} + 2^{-p}$.
Construct a deterministic B\"uchi automaton $\A^{\mu, d, \varepsilon}_{\geq 0}= (\State, \Start, \Sigma, \delta, \Final)$ as follows:
\begin{enumerate}
	\item  $\State = \{\thresh_l, \thresh_l+1,\dots,  \thresh_u \}$,
	 $\Start = \{0\}$ and
	 $\Final = \{\thresh_u\} $
	
	\item Alphabet $\Sigma = \{-\mu, -\mu+1,\dots, \mu-1, \mu\}$
	\item \label{Trans:formula} Transition function $\delta:  \State \times \Sigma \rightarrow \State$ s.t. $t = \delta(s,a)$ then:
	\begin{enumerate}
		
		\item~\label{item:inductive} If $s \in S\setminus \{\thresh_u,\thresh_l\}$ and $\roundL{d\cdot s \cdot 2^{-(p+k)} + a} = i \cdot 2^{-(p+k)}$ for $i\in\Z$
		\begin{enumerate}
			\item \label{Trans:IntState} If $\thresh_l \leq i \leq \thresh_u$, then $t = i$
		
			\item \label{Trans:leq} If $i> {\thresh_u}$, then $t = \thresh_u$
			\item \label{Trans:geq} If $i< {\thresh_l}$, then $t = \thresh_l$
		\end{enumerate}
	    \item \label{Trans:SelfLoop} Else, if $s = \thresh_l$ or $s= \thresh_u$, then $t = s$ for all $a \in \Sigma$
	\end{enumerate}
\end{enumerate}

Observe that the automaton is a co-safety automaton as its accepting state is a sink. It consists of $\O(\frac{\mu}{(d-1)^2\cdot \varepsilon})$ states.

We are left with the main proof that $\A^{\mu,d,\varepsilon}$
accepts an infinite weight sequence $W$ iff $\DSumL{W}\geq 0$.
For this, we explain the key ideas behind the construction. A state $s$ is interpreted to have a lower gap value of $s\cdot 2^{-(p+k)}$. Since the automaton is deterministic, every weight sequence, finite- or infinite-length, has a unique run in the automaton. so, Tthe idea is to ensure that for any finite-length weight sequence $A$ if state $s$ is the final state in its run in the automaton, then (a). if $s$ is $\thresh_u$, $\gapL{A} \geq \thresh_u\cdot 2^{-(p+k)}$, (b). if $s$ is $\thresh_l$, $\gapL{A} \leq \thresh_u\cdot 2^{-(p+k)}$, and (c)  $\gapL{A} = s\cdot 2^{-(p+k)}$ otherwise. 

In summary, the critical observation here is that Item~\ref{item:inductive} ensures that the transition function follows the inductive definition of lower gap from Definition~\ref{def:approxDSLow}.
This uses a proof by induction on the length of weight sequence $A$. 
If $|A|=0$, the final state of its run is the initial state 0, i.e., $\gapL{A} = 0$.  Suppose the hypothesis holds for weight-sequences of length $n$, we prove it holds for weight sequences of length $n+1$. Let $A = B\cdot b$ and $A$ be of length $n+1$. Then, suppose the final state in the run of $B$ is $s$. Suppose, $b \in S\setminus \{\thresh_u, \thresh_l\}$. 
Then, by I.H. $\gapL{B} = s\cdot 2^{-(p+k)}$. Let the automaton transition to state $t$ on reading alphabet $b$ from state $s$. Then, from definition of lower gap value, $\gapL{A} = \roundL{d\cdot\gapL{B} + b}$. In other words, $\gapL{A} = \roundL{d\cdot s \cdot 2^{-(p+k)} + b}$. This is precisely the criteria used in the transition function to determine the state $t$ in Eq.~\ref{Trans:formula}. Thus, suppose $\gapL{A} = i \cdot 2^{-(p+k)}$, then (a) if $\thresh_l\leq i \leq \thresh_u$, then $t = i$ and $\gapL{A} = t\cdot 2^{-(p+k)}$, (b) if $i > \thresh_u$ then $t = \thresh_u$ and $\gapL{A} = i\cdot 2^{-(p+k)} > t \cdot 2^{-(p+k)}$, and (c) if $i < \thresh_l$ then $t = \thresh_l$ and $\gapL{A} = i\cdot 2^{-(p+k)} < t \cdot 2^{-(p+k)}$. For the state $\thresh_u$, one can prove that if $\gapL{A} \geq \thresh_u\cdot 2^{-(p+k)}$ then for all $a \in \Sigma$, $\gapL{A\cdot a} \geq \thresh_u\cdot 2^{-(p+k)}$. Conversely, for the state $\thresh_l$, one can prove that if $\gapL{A} \leq \thresh_l\cdot 2^{-(p+k)}$ then for all $a \in \Sigma$, $\gapL{A\cdot a} \leq \thresh_l\cdot 2^{-(p+k)}$. This completes the proof of the claim.

Finally, to prove correctness it is sufficient to show that for all sequences $W$, $\DSumL{W} \geq$ 0 iff there exists a finite prefix $A$ of $W$ such that $\gapL{A} \geq \thresh_u\cdot 2^{-(p+k)}$. This is why state $\thresh_u$ is an accepting sink state. 
\end{proof}

\section{Case Study I: Grid World}

The human-robot interaction from is based off a classic $n\times n$ grid world domain. 
The human and robot correspond to the environment and system player. Initially, the two agents are present at diagonally opposite corners of the grid. Two bananas have been placed on the grid, one at each of the remaining corners. 
There are static obstacles of different configurations on these grids, e.g., placements of aisles (Fig~\ref{fig:fractional_df_req}) and an obstacle block in the center.  Each agent controls its own location and is allowed to move in the cardinal directions only. The agents take turns to change their location. We assume the human makes the first move. 
We say a collision occurs between the robot and an object/agent if the robot is in the same location as the object/agent. 
In this case, a strategy for the robot tells in which location to move to next based on the history of  previous configurations.

\renewcommand{\ban}{\mathsf{reach\_banana}}
\renewcommand{\wall}{\mathsf{collision\_obstacle}}
\renewcommand{\human}{\mathsf{collision\_human}}
The robot's hard (qualitative) constraint is to reach the location of at least one of the bananas without colliding into the static obstacles or the (moving) human. Thus, this constraint combines safety and reachability goals. It can be expressed as an $\ltl$ formula using atomic propositions $\ban$, $\wall$, and $\human$.
Proposition $\ban$ holds on those configurations of the grid in which the robot reaches the location of the banana. Proposition $\wall$ holds on  those configurations in which the robot collides with the wall. Similarly, proposition $\human$ holds on those configurations in which the robot collides with the human. 
Then, the $\ltl$ formula $\varphi$ is
\begin{align*}
    \varphi := & \ltlG (\neg \wall)  
    \wedge  \ltlG (\neg \human) \\ 
    \wedge & \ltlF (\ban)
\end{align*}

\renewcommand{\negr}{\mathsf{negative\_reward}}
\renewcommand{\posr}{\mathsf{positive\_reward}}

The robot's soft constraints are modelled to achieve two behaviors. The first one is to distance itself from the human. This could alternately be represented using temporal logic, however the representation will be cumbersome. 
Quantitative rewards can easily express this behavior. 
Given a negative integer parameter $\negr$, a negative reward is assigned to the  robot if it comes too close to the human. This is modelled using the Manhattan distance between the two agents. Suppose, the locations of the agents are $(x_0,y_0)$ and $(x_1,y_1)$, then given $\negr<0$, the reward received by the robot is  $$\Big\lfloor\frac{\negr}{|x_0-x_1| + |y_0-y_1|}\Big\rfloor$$

The second behavior expressed by soft constraints is to encode promptness to fulfil $\ban$. Temporal logics are good at specifying what should be done (using the $\ltlF$ operator) but, to the best of our knowledge, they cannot nicely specify measures such as promptness. One could attempt using several $\ltlX$ (Next operator) but that puts a hard bound on the number of steps within which the constraint must be satisfied. With quantitative constraints, one can encode promptness more naturally and softly (giving the robot more flexibility in deciding when to accomplish the constraint). In our case, we model promptness with a positive integer parameter $\posr>0$ which the robot receives only when it reaches a location of the banana for the first time. This is necessary since otherwise the robot's strategy could be to remain at the location of a banana, thus flouting the consideration to distance itself from the human. 

These two rewards are additive, i.e., if both the positive and negative rewards are non-zero in a configuration of the grid world, the robot receives the sum of both rewards in that configuration. 
Then, it is reasonable to say that a play accomplishes these two behaviors if the total discounted-sum reward of the robot is greater than or equal to 0, i.e., 0-satisficing plays/strategies are good for the robot. 
Observe that if the discount factor were an integer, then robot would be prompted to pick up the banana too soon. Then in Fig~\ref{fig:fractional_df_req} the robot would pick up the closer banana and would be unable to maintain sufficient distance from the human. With fractional discount factors, the robot recognizes it can plan for a longer term and will opt to reach the farther banana. This will also ensure it maintains distance form the human.  
This is exactly why fractional discount factors are preferred: they allow for planning on a longer term than what conservative integer factors would permit.

Our algorithm offers a method to soundly generate a strategy that is both $\varphi$-satisfying and 0-satisficing for the robot in this scenario.  The input to the algorithm will be a quantitative game $(G,\varphi, 0)$ where 
$G$ is a quantitative graph which formalizes the grid world, assigns its configurations (states) labels from the atomic propositions $\ban$, $\wall$, $\human$, and costs to transitions based on assignments from $\negr$ and $\posr$ as described above. 

The output of the algorithm is either a strategy for the robot which satisfies the $\ltl$ formula $\varphi$ and is 0-satisficing for the robot in the grid, or it is a strategy for the environment which either satisfies $\neg\varphi$ or is $d\cdot\varepsilon$-satisficing for the environment where $d$ and $\varepsilon$ are the discount factor and approximation factor, respectively.

\subsubsection{Empirical Analysis}
In the experiments on grid world, we take  $n= 4, 6,8,10$. We choose values of positive and negative rewards $(\posr, \negr)$ from the set $\{(5,-1), (10, -1), (10, -2), (20, -2), (20, -5)\}$, creating 20 grid world benchmarks.

\subsubsection{Observations and Inferences}

Our experiments demonstrate that our algorithm facilitates the design of provably correct strategies for the robot with respect to given the soft and hard constraints. This way we are able to soundly generate a strategy for the robot, from high-level specifications, which not only satisfies a temporal objective but also take into {\em softer} consideration  social-distancing and promptness. No other known approach is able to accomplish this task soundly.

Our algorithm solves all {\em all but one benchmark} within the timeout. The benchmark our algorithm failed on the largest grid of size $10\times 10$ when $d = 1.125$ $(k=3)$, $\posr = 20$, and $\negr = -2$.
The scalability trends of our algorithm on the grid world with a $2\times 2$ obstacle in the center of the grid on the $10\times 10$ grid with $\negr = -2$ have been summarized in Fig~\ref{fig:scalability}. 
The runtime trends with other grid sizes and negative values are similar. 
This shows that the performance of the algorithm is faithful to the size of the parity game which, in turn, is linear in the size of the comparator automata (Theorem~\ref{thrm:finalalgorithm}).

A thorough analysis of our experiments reveals avenues for improvement of the scalability of our algorithm. The one benchmark for which our algorithm failed to terminate within the timeout, we observed that the number of states in the product was high, the positive reward was high, and the discount factor was low ($10\times 10$ grid with $d = 1.125$,   $\posr = 20$, $\negr=-2$). Each one of these parameters contributes significantly to increasing the size of the comparator  (Theorem~\ref{thrm:Comparatorlower}) and subsequently the parity game (Theorem~\ref{thrm:finalalgorithm}). In this case, we observed that the algorithm ran out of memory on our machine. This suggests to focus on succinct representations of the comparator and the game in future work. 

Another observation has to do with the percentage of time spent in each step of the algorithm. Currently, our algorithm implements an explicit construction of the parity game. We observed that on most benchmarks, the algorithm spent around 70-80\% of its time constructing the parity game and only 20-30\% of the time in solving it. This indicates that another avenue for further scalability is to investigate approaches to solve parity games with decomposed specifications.

\begin{table}[t]
\label{tab:gridall}
\centering
\begin{tabular}{|c|c|c|c|}
\hline
\multicolumn{2}{|c|}{Rewards} &
  \multirow{2}{*}{\begin{tabular}[c]{@{}c@{}}Discount\\ factor\end{tabular}} &
  \multirow{2}{*}{\begin{tabular}{c}Total \\ time(s)\end{tabular}} \\
Positive & Negative & & \\ \hline
  
 \multicolumn{4}{|c|}{\textbf{Grid World} $n = 4$ with 397 states} \\ \hline
\multirow{2}{*}{5}  & \multirow{2}{*}{-1} & 1.25   & 0.015  \\ \cline{3-4} 
                    &                     & 1.125  & 0.049  \\ \hline
\multirow{2}{*}{10} & \multirow{2}{*}{-1} & 1.25   & 0.012  \\ \cline{3-4} 
                    &                     & 1.125 & 0.038  \\ \hline
\multirow{2}{*}{10} & \multirow{2}{*}{-2} & 1.25  & 2.084  \\ \cline{3-4} 
                    &                     & 1.125 & 11.770 \\ \hline
\multirow{2}{*}{20} & \multirow{2}{*}{-2} & 1.25  &  4.067  \\ \cline{3-4} 
                    &                     & 1.125 & 24.503 \\ \hline
\multirow{2}{*}{20} & \multirow{2}{*}{-5} & 1.25  &  4.542  \\ \cline{3-4} 
                    &                     & 1.125 & 25.850 \\ \hline

\multicolumn{4}{|c|}{\textbf{Grid World} $n = 6$ with 2407 states} \\ \hline
\multirow{2}{*}{5} & \multirow{2}{*}{-1} & 1.25 &        0.050     \\ \cline{3-4}
                   &                     & 1.125 &      0.158     \\ \hline
\multirow{2}{*}{10} & \multirow{2}{*}{-1} & 1.25 &     0.050     \\ \cline{3-4}
                   &                     & 1.125 &       0.150     \\ \hline
\multirow{2}{*}{10} & \multirow{2}{*}{-2} & 1.25 &        6.856     \\ \cline{3-4}
                   &                     & 1.125 &         46.199     \\ \hline
\multirow{2}{*}{20} & \multirow{2}{*}{-2} & 1.25 &     13.987     \\ \cline{3-4}
                   &                     & 1.125 &      94.525     \\ \hline
\multirow{2}{*}{20} & \multirow{2}{*}{-5} & 1.25 &     19.424     \\ \cline{3-4}
                   &                     & 1.125 &     106.136     \\ \hline

\multicolumn{4}{|c|}{\textbf{Grid World} $n = 8$ with 8093 states} \\ \hline
 \multirow{2}{*}{5} & \multirow{2}{*}{-1} & 1.25 &   0.159      \\ \cline{3-4}
                   &                     & 1.125  &   0.444    \\ \hline
\multirow{2}{*}{10} & \multirow{2}{*}{-1} & 1.25 &     0.158      \\ \cline{3-4}
                   &                     & 1.125 &       0.419      \\ \hline
\multirow{2}{*}{10} & \multirow{2}{*}{-2} & 1.25 &   19.952     \\ \cline{3-4}
                   &                     & 1.125 &    138.201     \\ \hline
\multirow{2}{*}{20} & \multirow{2}{*}{-2} & 1.25 &     38.293     \\ \cline{3-4}
                   &                     & 1.125 &     279.519     \\ \hline
 \multirow{2}{*}{20} & \multirow{2}{*}{-5} & 1.25 &     56.544     \\ \cline{3-4}
                   &                     & 1.125 &    330.451     \\ \hline

\multicolumn{4}{|c|}{\textbf{Grid World} $n = 10$ with 20572 states} \\ \hline
 \multirow{2}{*}{5} & \multirow{2}{*}{-1} & 1.25 &    0.416     \\ \cline{3-4}
                   &                     & 1.125 &    0.972     \\ \hline
\multirow{2}{*}{10} & \multirow{2}{*}{-1} & 1.25 &      0.413     \\ \cline{3-4}
                   &                     & 1.125 &      0.914     \\ \hline
\multirow{2}{*}{10} & \multirow{2}{*}{-2} & 1.25 &     39.102     \\ \cline{3-4}
                   &                     & 1.125 &  315.064     \\ \hline
\multirow{2}{*}{20} & \multirow{2}{*}{-2} & 1.25 &    78.329     \\ \cline{3-4}
                   &                     & 1.125 &  Timeout     \\ \hline
 \multirow{2}{*}{20} & \multirow{2}{*}{-5} & 1.25 &    122.792     \\ \cline{3-4}
                  &                     & 1.125 &  Timeout     \\ \hline

\end{tabular}
\caption{Analysis of Grid World Domain. Table does not record $d=1.5$ to improve readability of table. All runs with $d=1.5$ terminated within less than $1 \sec$. Timeout = 750$\sec$}
\end{table}

\section{Case Study II: Conveyor Belt}
In our second case study, we consider a significantly more challenging set of scenarios. A robot must operate along a $r\times c$ conveyor belt with $r$ rows and $c$ columns across from a human, see Fig~\ref{fig:3conveyor}. Both agents are restricted to not reach fully across the conveyor belt. When out of reach of the human, the robot can move quickly. Otherwise, it must proceed more slowly. The blocks move down the conveyor belt at a constant speed.

\begin{figure}[t]
\centering
\includegraphics[width=0.3\textwidth]{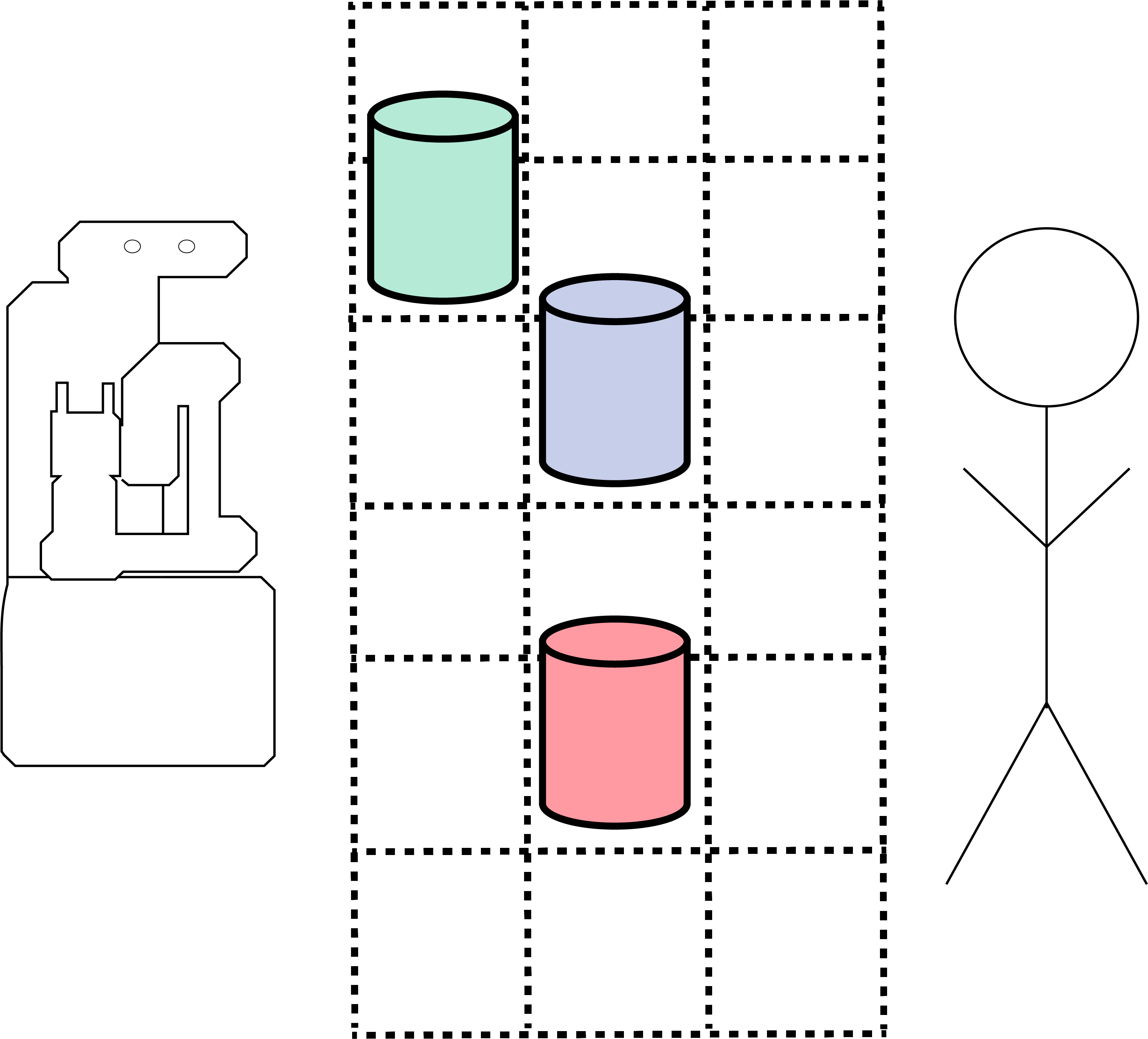}
\caption{Example conveyor belt scenario with three blocks.}
\label{fig:3conveyor} 
\end{figure}

The human controls the location of its arm and the placement of new objects. New blocks of identical type (color) are placed whenever a block is removed from the belt so that a constant number and proportion of types of blocks are maintained on the belt. The human controls the placement of new blocks, except that it must place green blocks near the robot (to ensure the game is winnable).

\paragraph{2 blocks.}
\newcommand{\collision}{\mathsf{collision}}
\newcommand{\blockh}{\mathsf{block\_human}}
\newcommand{\drop}{\mathsf{dropped\_critical}}

In the two block scenario, the robot's $\ltl$ goal is to ensure it doesn't interfere with the human grasping objects. We define proposition $\collision$ as in the previous example. Proposition $\blockh$ holds in a state if the robot and human are adjacent to the human's object and the human simultaneously. Then, the robot's $\ltl$ goal in the 2-block scenario is

\[
    \varphi_2 :=  \ltlG (\neg \collision)  
    \wedge  \ltlG (\neg \blockh)
\]

The robot's soft constraint is designed to encourage it to pick up as many blocks as possible. The robot receives a positive reward for every block it picks up and a negative reward for every block that falls off the belt.

\paragraph{3 blocks.}
In the three block scenario (Fig~\ref{fig:3conveyor}), 
the green blocks are ``critical'' and the robot must grab one. The blue blocks are ``desired'' and the robot should retrieve as many of them as possible. The red blocks are ``the human's'' and the robot should ensure it never blocks the human from reaching them. 
The robot's $\ltl$ goal is to ensure it grasps all green objects and avoids  the human grasping red objects. We define Propositions $\collision$ and $\blockh$ as in the previous example. Proposition $\drop$ holds if a critical object has been dropped prior to or in the current state.

\begin{align*}
    \varphi_3 := & \ltlG (\neg \collision)  
    \wedge  \ltlG (\neg \blockh) \\ 
    \wedge & \ltlG (\neg \drop)
\end{align*}

The robots soft constraint is to maximize the number of blue objects it grasps. Each arm is modeled as grid cells emanating from either side of the conveyor belt. The robot controls the location of its arm. Every desired object retrieved gives positive reward. Every desired object that falls off the end of the belt gives negative reward. If both positive and negative reward are achieved in the same step, the rewards are added.

\subsection{Empirical Evaluation}
In the experiments on conveyor belt, $r\times c = 4\times 3, 5\times 3$ with 2 or 3 blocks. We choose $\negr = -1$. With 2 blocks, we choose $\posr = 2, 3, 4$  and with 1 block $\posr = 5$, creating 7 conveyor belt benchmarks.

\subsubsection{Observations and Inferences}

For the two block scenario, our algorithm solves all but one benchmark. The failure here is a $5\times 3$ conveyor belt  when the positive reward is $4$ and the discount factor is $1.125$ $(k=3)$. 
As earlier, the runtime trends are consistent with the theoretical analysis on size of the parity game and the comparator.

In the solved cases, we see that the algorithm generates a strategy for the robot in all games (which we engineer so that the robot can win). We note that the robot quickly obtains its rewards, suggesting its policy is of high-quality.
Unfortunately, the complexity of the game makes it intractable to hand-compute an optimal policy and compare it to the robot's policy generated by the algorithm. The inability to manually or algorithmically check the correctness of a policy w.r.t. optimality is a reason why one would want sound algorithms like ours to solve complex scenarios like this.  

On the three block scenario, we performed experiments on the $5\times 3$ conveyor belt. Our algorithm terminates on the belts when the discount factor is $d=1.5, 1.25$ $(k=1, k=2)$ but it struggled with discount factor $d=1.125 (k=3)$. As a representative case. Further, none of our experiments terminated at $d = 1.125$. This is not surprising since the product game is large (~60K states) and the discount factor is low. 
Again, we see that future work will require improved scalability. This will open up new applications for robotic synthesis

\begin{table}[t]

\label{table:conveyor_3x4}
\centering
\begin{tabular}{|c|c|c|c|}
\hline
\multicolumn{2}{|c|}{Rewards} &
  \multirow{2}{*}{\begin{tabular}[c]{@{}c@{}}Discount\\ factor\end{tabular}} &
  \multirow{2}{*}{\begin{tabular}{c}Total \\ time(s)\end{tabular}} \\
Positive & Negative & & \\ \hline
\multicolumn{4}{|c|}{\textbf{Conveyor Belt  $r\times c = 4\times3$ with 2 blocks (9966 states)}} \\ \hline
\multirow{2}{*}{2} & \multirow{2}{*}{-1} & 1.25 &     20.121     \\ \cline{3-4}
                   &                     & 1.125 &    102.815     \\ \hline
\multirow{2}{*}{3} & \multirow{2}{*}{-1} & 1.25 &     29.922     \\ \cline{3-4}
                   &                     & 1.125 &    152.274     \\ \hline
\multirow{2}{*}{4} & \multirow{2}{*}{-1} & 1.25 &     40.216     \\ \cline{3-4}
                   &                     & 1.125 &    208.748     \\ \hline
                   
\multicolumn{4}{|c|}{\textbf{Conveyor Belt  $r\times c = 5\times3$ with 2 blocks ( 31547 states)}} \\ \hline
\multirow{2}{*}{2} & \multirow{2}{*}{-1} & 1.25 &     64.782    \\ \cline{3-4}
                   &                     & 1.125 &    332.764     \\ \hline
\multirow{2}{*}{3} & \multirow{2}{*}{-1} & 1.25 &     98.520     \\ \cline{3-4}
                   &                     & 1.125 &    677.558     \\ \hline
\multirow{2}{*}{4} & \multirow{2}{*}{-1} & 1.25 &     127.422     \\ \cline{3-4}
                   &                     & 1.125 &    Timeout     \\ \hline
                   
\multicolumn{4}{|c|}{\textbf{Conveyor Belt $r\times c = 5\times3$ with 3 blocks (  60540 states)}} \\ \hline
\multirow{2}{*}{5} & \multirow{2}{*}{-1} & 1.25 &     712.941     \\ \cline{3-4}
                   &                     & 1.125 &    Timeout     \\ \hline
\end{tabular}
\caption{Analysis of Conveyor Belt domain. Table does not record $d=1.5$ to improve readability of table. All runs with $d=1.5$ terminated within less than $10 \sec$. Timeout = 750$\sec$}
\label{table:conveyorall}
\end{table}

\begin{table}[t]
\caption{10x10 social dist with varied approximation factor}
\label{table:approx_factor}
\centering
\begin{tabular}{|c|c|c|c|c|}
\hline
\multicolumn{2}{|c|}{Rewards} &
  \multirow{2}{*}{\begin{tabular}[c]{@{}c@{}}Disc.\\ factor\end{tabular}} &
  \multirow{2}{*}{\begin{tabular}[c]{@{}c@{}}Approx\\ factor\end{tabular}} &
  \multirow{2}{*}{\begin{tabular}[c]{@{}c@{}}Total\\ time(s)\end{tabular}} \\
Pos & Neg & & & \\ \hline
\multirow{6}{*}{5} & \multirow{6}{*}{-1} &   \multirow{2}{*}{1.5} & 1.25 &      0.385     \\ \cline{4-5}
                   &                     &  &       1.125           &      0.394     \\ \cline{3-5}
                   &                     & \multirow{2}{*}{1.25} & 1.25 &      0.415     \\ \cline{4-5}
                   &                     &  &       1.125           &      0.421     \\ \cline{3-5}
                   &                     & \multirow{2}{*}{1.125} & 1.25 &      0.915     \\ \cline{4-5}
                   &                     &  &       1.125           &      0.917     \\ \hline
\multirow{6}{*}{5} & \multirow{6}{*}{-2} &  \multirow{2}{*}{1.5} & 1.25 &      1.642     \\ \cline{4-5}
                   &                     &  &       1.125           &      6.407     \\ \cline{3-5}
                   &                     & \multirow{2}{*}{1.25} & 1.25 &     39.018     \\ \cline{4-5}
                   &                     &  &        1.125          &     82.824     \\ \cline{3-5}
                   &                     & \multirow{2}{*}{1.125} & 1.25 &    319.686     \\ \cline{4-5}
                   &                     &  &         1.125          & ---     \\ \hline
\multirow{6}{*}{10} & \multirow{6}{*}{-1} & \multirow{2}{*}{1.5} & 1.25 &      0.383     \\ \cline{4-5}
                   &                     &  &         1.125         &      0.443     \\ \cline{3-5}
                   &                     & \multirow{2}{*}{1.25} & 1.25 &      0.418     \\ \cline{4-5}
                   &                     &  &       1.125           &      0.428     \\ \cline{3-5}
                   &                     & \multirow{2}{*}{1.125} & 1.25 &      0.908     \\ \cline{4-5}
                   &                     &   &       1.125           &      0.917     \\ \hline
\multirow{6}{*}{10} & \multirow{6}{*}{-2} &  \multirow{2}{*}{1.5} & 1.25 &      4.001     \\ \cline{4-5}
                   &                     &   &        1.125          &     11.082     \\ \cline{3-5}
                   &                     & \multirow{2}{*}{1.25} & 1.25 &     80.069     \\ \cline{4-5}
                   &                     &  &         1.125         &    160.947     \\ \cline{3-5}
                   &                     & \multirow{2}{*}{1.125} & 1.25 & ---     \\ \cline{4-5}
                   &                     &  &      1.125             & ---     \\ \hline

\end{tabular}
\end{table}

\end{document}